%% file: main_arxiv.tex
\documentclass[11pt,letterpaper]{article}

 \usepackage[hscale=0.75,vscale=0.8]{geometry}

\usepackage[utf8]{inputenc} % allow utf-8 input
\usepackage[T1]{fontenc}    % use 8-bit T1 fonts
\usepackage{hyperref}       % hyperlinks
\usepackage{url}            % simple URL typesetting
\usepackage{booktabs}       % professional-quality tables
\usepackage{amsfonts}       % blackboard math symbols
\usepackage{nicefrac}       % compact symbols for 1/2, etc.

\usepackage{microtype}
\usepackage[nosetpagesize]{graphicx}
\usepackage{wrapfig}
\usepackage{subcaption}
\usepackage{amssymb,amsmath,amsthm,color,bm,algorithm,algorithmic}

\usepackage[linesnumbered,ruled,vlined,algo2e]{algorithm2e}

\usepackage[linesnumbered,ruled,vlined,algo2e]{algorithm2e}

\SetCommentSty{mycommfont}

\SetKwInput{KwInput}{Input}                % Set the Input
\SetKwInput{KwOutput}{Output}              % set the Output

\usepackage{url}

\SetCommentSty{mycommfont}

\usepackage[mathcal]{eucal}

\DeclareMathAlphabet\mathbfcal{OMS}{cmsy}{b}{n}

\newcommand{\BEAS}{\begin{eqnarray*}}
\newcommand{\EEAS}{\end{eqnarray*}}
\newcommand{\BEA}{\begin{eqnarray}}
\newcommand{\EEA}{\end{eqnarray}}
\newcommand{\BEQ}{\begin{equation}}
\newcommand{\EEQ}{\end{equation}}
\newcommand{\BIT}{\begin{itemize}}
\newcommand{\EIT}{\end{itemize}}
\newcommand{\BNUM}{\begin{enumerate}}
\newcommand{\ENUM}{\end{enumerate}}
\newcommand{\BA}{\begin{array}}
\newcommand{\EA}{\end{array}}

\newcommand{\argmax}{\mathop{\rm argmax}}

\newcommand{\sign}{\mathop{ \rm sign}}

\newcommand{\rb}{\mathbb{R}}

\newtheorem{lemma}{Lemma}
\newtheorem{theorem}{Theorem}
\newtheorem*{theorem*}{Theorem}

\def \E{{\mathbb E}}
\def \P{{\mathbb P}}

\def \P{{\mathbb P}}

\allowdisplaybreaks[4]
\input{notation}

\usepackage{cleveref}

\usepackage[square,sort,comma,numbers]{natbib}

\title{Convergence of Uncertainty Sampling for Active Learning} 

\author{Anant Raj \\
  Inria, Ecole Normale Sup\'erieure \\
  PSL Research University, Paris, France. \\
  \texttt{anant.raj@inria.fr} 
\and Francis Bach \\
  Inria, Ecole Normale Sup\'erieure \\
  PSL Research University, Paris, France. \\
  \texttt{francis.bach@inria.fr} \\
}

\begin{document}

\maketitle

\begin{abstract}
Uncertainty sampling in active learning is heavily used in practice to reduce the annotation cost. However, there has been no wide consensus on the function to be used for uncertainty estimation in binary classification tasks and convergence guarantees of the corresponding active learning algorithms are not well understood. The situation is even more challenging for multi-category classification. In this work, we propose an efficient uncertainty estimator for binary classification which we also extend to  multiple classes, and provide a non-asymptotic rate of convergence for our uncertainty sampling based active learning algorithm in both cases under no-noise conditions (\textit{i.e}., linearly separable data). We also extend our analysis to the noisy case and provide theoretical guarantees for our algorithm under the influence of noise in the task of binary and multi-class classification.

\end{abstract}

\section{Introduction}\label{sec:intro}
Over the last decade, machine learning algorithms have achieved a lot of success on various tasks in computer vision, natural language processing, and speech recognition. This success has been led by various factors which include improvement in computing architectures and improved machine learning algorithms. Moreover, the rapid growth in the number of large labeled public datasets is also one of the most important factors which contributed to the rise of  machine learning. However, in many practical scenarios, the labeled data are hard to obtain, as it requires a lot of time and human efforts to label a dataset. Hence, it is a time consuming as well as economically expensive procedure to perform. For this reason, there have been a lot of efforts to build machine learning algorithms which require a significantly lower number of labeled samples to train. One major direction of research in this area is to devise efficient \emph{active learning} algorithms. 

Active learning algorithms propose efficient labeling schemes to reduce the number of labels required in order to train a classifier, resulting in  minimal   annotation cost while still maintaining high performance. Active learning methods can be categorized in   two major categories : (i) stream-based active learning where samples from the data generating distribution are sequentially presented to the active learner, and (ii) pool-based active learning where there exists a very small number of labeled samples and the rest of the samples have no label. We note that any streaming based active learning algorithm can be converted to a pool-based active learning algorithm and vice versa~\cite{sabato2016interactive}.  For both of these categories, there exists an acquisition function which characterizes which informative samples should be labelled. The most popular way of determining if a sample is informative or not  is by estimating uncertainty. Other than uncertainty sampling based approaches, the other major approaches for performing active learning  are query-by-committee \cite{seung1992query}, expected model change \cite{settles2007multiple}, expected error reduction \cite{roy2001toward},  expected variance reduction \cite{wang2015ambiguity}, among others. 

The main focus of this paper is \emph{uncertainty sampling} based active learning algorithms. 
Despite of it being one of the most used active learning algorithms in practice, little is known about its theoretical properties. Specifically, there has been no common consensus on  the optimal uncertainty estimation approach used to perform active learning and its convergence properties. Also, most of the active learning algorithms studied previously are for binary classification and are not trivial to extend to multi-category classification problems. In this paper, we investigate these questions from the lens of optimization through stochastic gradient descent, and propose a sampling function for which the uncertainty sampling based active learning algorithm provably converges. We make the following contributions in this work :
\begin{itemize}
    \item[(i)] We propose a family of theoretically motivated functions to estimate uncertainty for linear  predictions for  binary and multi-category classification.
    
    \item[(ii)] We show that the active learning algorithm based on the proposed sampling scheme converges to the optimal predictor in the separable regime for binary classification, and can be easily extended to  multi-category classification. We provide a non-asymptotic rate of convergence of order $ {O}(1/n)$, where $n$ is number of iterations of the algorithms which is also number of unlabeled samples seen by the algorithm for both binary and multi-class cases. 
    
    \item[(iii)] We extend our analysis to the inseparable regime for both  cases (binary and multi-class) and show that the probability of mis-classification is bounded by ${O}(1/n)+{O}(\eta)$  where $n$ is number of iteration performed by the algorithms and $\eta$ is the noise parameter. 
    \item[(iv)] We perform experimental evaluations for our algorithm on classification tasks.
\end{itemize}

%Major active learning approaches can be categorized in these major categories: (i) Uncertainty sampling (ii)

\subsection{Related Work}\label{sec:rel_work}
There has been a vast amount of work done in the field of active learning and we only give a brief overview here. For binary classification, online active learning has been studied under the name of selective sampling under adversarial assumptions  \cite{cesa2009robust,dekel2010robust,orabona2011better,cavallanti2011learning}. However, these methods can not be extended to multi-class classification and are computationally expensive.  \citet{agarwal2013selective} proposed a selective sampling scheme for multi-class classification for generalized linear models, but still each step of the algorithm is computationally expensive to perform. \citet{settles2012active} provides an excellent survey of empirical studies in active learning. 

Apart from empirical studies, there has been a lot of works in active learning on the theoretical front. Disagreement based active learning has been an active area of research in machine learning. The primary idea is as follows. A set of possible empirical risk minimizers are maintained with time and a label is queried if two minimizers disagree on the predicted label of that sample. An excellent survey of disagreement based active learning is provided in \cite{hanneke2014theory}. 

Our work can also be related to the vast line of work done in margin based active learning \cite{balcan2007margin,dasgupta2005analysis,balcan2013active,wang2016noise}. However for most of works in this area, the gain and convergence of the algorithm have only been shown under strong distributional assumptions while we do not make any such assumption on the data for our uncertainty sampling based active learning algorithm, and still manage to show a non-asymptotic rate of convergence for the algorithm. 

Uncertainty sampling based machine learning algorithms have a long history. They were first proposed by \citet{lewis1994sequential} who experimentally show that a probabilistic model with uncertainty sampling can improve the performance of text classification by up to 500 fold. Later, \citet{schohn2000less} applied uncertainty sampling to SVM classification and showed improved performance. Since then, it has been widely used for performing active learning \cite{yang2015multi,zhu2008active,lughofer2017online,yang2016active,wang2017uncertainty}. However, none of  the above mentioned works focuses on the theoretical understanding of uncertainty sampling. Recently, \citet{mussmann2018uncertainty} showed that threshold based uncertainty sampling on a convex (e.g., logistic) loss can be interpreted as performing a preconditioned
stochastic gradient step on the population zero-one loss. However, a proper convergence analysis was missing in \cite{mussmann2018uncertainty}, in part due to the underlying non-convexity of their formulation.  
 
\section{Background}\label{sec:background}
\subsection{Uncertainty Sampling} \label{sec:uncertain}
Because of the ease of application, uncertainty sampling remains one of the most popular approaches used to perform active learning. Uncertainty sampling relies on the idea of querying the data point about which the current predictor is most uncertain. In simpler terms, 
uncertainty sampling usually identifies those points which are  close to the decision boundary of the current model. However, the most important task here is to compute the uncertainty of prediction. There have been several approaches proposed to measure the uncertainty of a prediction. Here below,  we discuss few of them that are widely used in practice \cite{monarch2021human}. Let us assume that a probabilistic model generates predictions in the form of probability distributions $p_{\theta}(\cdot|x)$ on $\mathcal{Y}$ for $x\in \mathcal{X}$ and model parameter $\theta$.

\textbf{Margin of confidence sampling \cite{monarch2021human,nguyen2021measure}.} An intuitive way to estimate the uncertainty is by computing the margin in the confidence of top two predictions. Mathematically, sampling probability of querying a label $p_u(x,\theta) = \sigma(p_{\theta}(y_1^\star|x)- p_{\theta}(y_2^\star|x))$ where $\sigma:\mathbb{R} \rightarrow [0,1]$, and  $y_1^\star$ and $y_2^\star$ correspond to the two top most predictions for $x$ given the model parameter $\theta$.

\textbf{Least confidence sampling \cite{monarch2021human,nguyen2021measure}.}   Least confidence sampling considers the difference between 100\% confidence and the most confident prediction to compute sampling probability of query a label. That means $p_u(x,\theta) \propto 1- p_{\theta}(y_1^\star|x)$ where $y_1^\star$ corresponds to the  top most prediction for $x$ given the model parameter $\theta$.

\textbf{Entropy-based sampling \cite{monarch2021human,nguyen2021measure}.} Entropy is an information theoretically motivated way to compute the uncertainty and widely used to estimate the uncertainty. Sampling probability of querying  a label can be written as: $p_u(x,\theta) \propto \sum_{y\in \mathcal{Y}} p_{\theta} (y|x) \log p_{\theta} (y|x) $. 

In this paper, we use the margin of confidence sampling scheme to estimate uncertainty in the prediction. Details of the sampling function $\sigma$ will be provided in  \cref{sec:convergence} where we discuss convergence of the algorithm. However, before discussing the theoretical results (\cref{sec:convergence}), in the next section we discuss the relation between the hinge loss and corresponding test accuracy in binary and multi-class classification. 

\subsection{Max-margin linear classification} \label{sec:hinge}

In this paper, we consider the simplest possible set-up of linear classification, with inputs $x \in\rb^d$, and linear prediction functions. We note that by replacing $x$ by some feature function $\Phi(x)$ we can deal with non-linear problems, the feature map being explicit or implicit through kernel methods~\cite{hofmann2008kernel}. 

\textbf{Binary classification.} With two classes, we consider $y \in \{-1,1\}$ and a prediction function $x \mapsto \theta^\top x$ parameterized by $\theta \in \rb^d$. We then classify according to the sign of $\theta^\top x$.

The associated error rate can be computed as $$\P( y\theta^\top x \leqslant 0),$$ but is is a non-convex function of $\theta$. Among the many convex surrogates, we will consider the classical hinge loss:
$$
\hat{\ell}(x,y,\theta) = \max\{0, 1 - y \theta^\top x\},
$$
and its square, leading to the traditional support vector machine. The regular hinge loss is non differentiable, while the squared hinge loss is smooth in $\theta$. In this paper, we will consider an algorithm based on the squared hinge loss, but obtain a guarantee for the non-squared one. Given that the two losses lead to guarantees on the misclassification error, as
$$
\P( y\theta^\top x \leqslant 0) \leqslant \E(\hat{\ell}(x,y,\theta))$$
$$\P( y\theta^\top x \leqslant 0) \leqslant \E(\hat{\ell}(x,y,\theta)^2),
$$
this allows to get the desired bounds.

\textbf{Multi-class classification.} Here, we consider $y \in \{1,\dots,k\}$. In multi-class classification the model parameter $\theta$ is a vector in $\mathbb{R}^{d k}$ which consists of predictors $\theta(i) \in \mathbb{R}^d$ for all $i \in \{1,2,\ldots,k \}$. Hence, we denote model parameter $\theta$ as collection  of $k$ predictors, \textit{i.e.}, $\theta = [\theta(1); \theta(2); \cdots ;\theta(k)]$. We consider the multi-class SVM formulation of \cite{crammer2001algorithmic}.  Following the structured SVM notation in \cite{tsochantaridis2005large}, let us assume that $\phi(x,y)$ represents the feature map corresponding to the sample pair $(x,y)$. In multi-class classification with $k$ classes, $\phi(x,y) \in \mathbb{R}^{dk}$ consists of $k$ blocks of $d$-dimensional vector and if we allow ourselves to denote each $d$-dimensional block with $\phi(x,y)(i)$ for $i \in \{1,2,\cdots,k\}$, then $\phi(x,y) = [\phi(x,y)(1);\phi(x,y)(2);\cdots;\phi(x,y)(k)]$ where $\phi(x,y)(i) = 0$ for all $i\neq y$ and $\phi(x,y)(y)= x$. Define a loss function $\Delta : \mathcal{Y}\times \mathcal{Y}\rightarrow \mathbb{R}$. General hinge loss for maximum margin training  function  can be written as :
\begin{align*}
    \hat{\ell}(x_t,y_t, \theta) = \max_{y\in \mathcal{Y}} \left[ \Delta(y,y_n) - \theta^\top (\phi(x_n,y_n) - \phi(x_n,y)) \right].
\end{align*}
In the  case of  multi-class classification, generally   $\Delta(y,y') = 1$ if $y\neq y'$, otherwise $\Delta(y,y') = 0$. Hence, the multi-class hinge loss can be written as,
\begin{align}
    &\hat{\ell}(x,y, \theta)  = \max \left[ 0, 1 - \theta^{\top} \left(\phi(x,y) - \phi(x,y^\star(\theta,x,y)\right) \right],  \label{eq:multiclass-hinge}\\
    &\text{where}~~ y^\star(\theta,x,y) = \argmax_{z\in {\mathcal{Y}\backslash {y}}} \theta^\top \phi(x,z). \notag
\end{align}
Predicting a label for the data point $x \in \mathbb{R}^d$ by the predictor $\theta$ is done by computing $\argmax_{z\in {\mathcal{Y}}} \theta^\top \phi(x,z)$. Similar to the binary case, the hinge loss for multiclass classification is non differentiable while the square hinge loss is smooth and differentible in model parameter $\theta$. For  multi-class hinge loss as well, the misclassification error can be bounded by expected loss for both the losses. That is for a given sample pair $(x,y)$ and model parameter $\theta$, we have,
\begin{align*}
    \P( \theta^\top \phi(x,y) - \theta^\top \phi(x,z^\star) \leqslant 0) \leqslant \E(\hat{\ell}(x,y,\theta)) \\
\P( \theta^\top \phi(x,y) - \theta^\top \phi(x,z^\star) \leqslant 0) \leqslant \E(\hat{\ell}(x,y,\theta)^2),
\end{align*}
where $z^\star = \argmax_{z\in {\mathcal{Y}}} \theta^\top \phi(x,z)$.

\section{Convergent Uncertainty Sampling for Classification} \label{sec:convergence}
\begin{algorithm}[h]
\LinesNumberedHidden
%\DontPrintSemicolon

  \KwInput{learning rate $\gamma$, Streaming $(x_i,y_i)$ for $i\in [n]$, initial model parameter $\theta_1$,  parameter $\mu$ and sampling function $\sigma$.}\;

  \KwOutput{average iterate $\bar{\theta}_{n+1}$. }
  
\For{$t\gets 1$ \KwTo $n$ }
{   \hspace{0.5mm}
    Compute probability $p_u(x_t,\theta_t) = \sigma(\theta_t, x_t).$ \\
    Sample Bernoulli random variable  $z_t$ with $ p=p_u(x_t,\theta_t).$ \\
    Compute $\hat{\ell}_t(x_t,y_t,\theta_t) \gets  \max(0, 1 - y_t (\theta_t^\top x_t)).$ \\
    Update $\theta_{t+1}\gets \theta_t + \gamma z_t (y_t x_t) \hat{\ell}_t(x_t,y_t,\theta_t).$  \\
    Update $\bar{\theta}_{t+1} \gets \left(1 - \frac{1}{t+1} \right)\bar{\theta}_{t+1} + \frac{1}{t+1} \theta_{t+1}.$
    }
\caption{ Uncertainty Sampling in Binary Classification}\label{alg:uncertain_binary}
\end{algorithm}

In this work, we consider the streaming data setting. However, the algorithm can also be applied for non-streaming data setting. In the next three sections, we would discuss the convergence results for binary and multi-class classification.

Let us consider $n$ \textit{i.i.d.}~samples $(x_i,y_i)$ jointly sampled from $\mathcal{P}$ such that $x_i \in \rb^d$, $i=1,\dots,n$, and  $y_i\in \mathcal{Y}$ where $\mathcal{Y} = \{ -1,1\}$ for binary classification and $\mathcal{Y} =  \{1,2,\ldots, k\}$ for multi-class classification.  

Before going into the details of theoretical results, we discuss the intuition of the uncertainty sampling algorithm here below. As discussed previously, the main idea behind the uncertainty sampling based active learning algorithm is that to query the labels for those data point about which the predictor is uncertain about.  In streaming data setting, every-time we see a new data point, we compute its uncertainty by computing the prediction score. Then, we convert this score into the density function with the help of the given function $\sigma$ which maps prediction score to probability of querying a label. This probability score is used to generate a Bernoulli random variable which decides if the algorithm would query  the label of the presented sample or not, and then perform a stochastic gradient step (for the squared hinge loss) if the label is accessed. The exact expression for the function $\sigma$ would be provided in \cref{sec:binary} (for binary classification) and in \cref{sec:multiclass} for multiclass classification. The pseudo-codes of the algorithms are presented in \cref{alg:uncertain_binary} (binary classification) and \cref{alg:uncertain_multi-class} (multi-class classification). %\note{FB: shouldn't we cite the paper of Liang here as they do the same?} \note{Anant: They perform more aggressive sampling. }

\begin{algorithm}[h]
\LinesNumberedHidden
\DontPrintSemicolon

  \KwInput{learning rate $\gamma$, Streaming $(x_i,y_i)$ for $i\in [n]$, initial model parameter $\theta_1$,  parameter $\mu$, number of classes k, and sampling function $\sigma$.}\;

  \KwOutput{average iterate $\bar{\theta}_{n+1}$. }
  
\For{$t\gets 1$ \KwTo $n$ }
{   \hspace{0.5mm} Compute score $s_t(j) = \theta(j)^\top x_t$ for all $j \in [k]$ . \\
    Compute probability $p_u(x_t,\theta_t) = \sigma(\theta_t, x_t).$ \\
    Sample Bernoulli random variable  $z_t$ with $ p=p_u(x_t,\theta_t).$ \\
    $y_t^\star \gets \argmax_{j\in \mathcal{Y}\backslash y_t} s_t(j)$\\
    $\hat{\ell}_t(x_t,y_t,\theta_t) \gets  \max(0, 1 - \theta_t(y_t)^\top x_t + \theta_t(y_t^\star)^\top x_t).$ \\
    Update $\theta_{t+1}(y_t) \gets \theta_{t}(y_t) + \gamma z_t  x_t \hat{\ell}_t(x_t,y_t,\theta_t).$   \\
    Update $\theta_{t+1}(y_t^\star) \gets \theta_{t}(y_t^\star) - \gamma z_t  x_t \hat{\ell}_t(x_t,y_t,\theta_t).$ \\
    Update $\theta_{t+1} \gets [\theta_{t+1}(1); \theta_{t+1}(2);\ldots; \theta_{t+1}(k)].$ \\
    Update $\bar{\theta}_{t+1} \gets \left(1 - \frac{1}{t+1} \right)\bar{\theta}_{t+1} + \frac{1}{t+1} \theta_{t+1}.$
    }
\caption{ Uncertainty Sampling in Multi-Class Classification}\label{alg:uncertain_multi-class}
\end{algorithm}

\subsection{Binary Classification} \label{sec:binary}
In the separable case for binary classification we, assume that there exists an optimal classifier $\theta_\star \in \mathbb{R}^d$ such that for all $x\in \mathcal{X}$ and its corresponding label $y\in \mathcal{Y}$ where $\mathcal{Y} = \{ -1,1\}$,
\begin{align}
    y(\theta_\star^\top x) \geq \rho^\star > 1 .
\end{align}
The above assumption is a standard assumption made in the analysis of maximum margin classifier in the realizable case \cite{balcan2007margin,dasgupta2005analysis}. 
We  minimize the square hinge loss to obtain our predictor. The expression for square hinge loss can be written as (with an extra factor of $1/2)$, 
\begin{align}
    \ell(x,y,\theta)=  \frac{1}{2}\max\left[0, 1 - y\theta^\top x\right]^2.
\end{align}
We have the following update rule to update the model parameter $\theta$ for uncertainty sampling based active learning :
\begin{align}
    \theta_{t+1} = \theta_t  + \gamma~z_t(y_t x_t)\big[1 -y_t({\theta_t}^\top x_t) \big]_{+}, \label{eq:update_bianry}
\end{align}
where $z_t$ is a Bernoulli  random variable for fixed $\theta_t$ and $x_t$ such that $p(z_t=1|x_t,\theta_t) = \sigma(\theta_t, x_t)$ where $\sigma:\mathbb{R}^d\times \mathbb{R}^d\rightarrow [0,1]$ is an even function. 

When $\sigma = 1$, \textit{i.e.}, querying label for every sample, then this is exactly stochastic gradient descent for the squared hinge loss, which is known to converge with rate $O(1/t)$ in the separable situation~\cite{vaswani2018fast}. In the next result we show that the algorithm proposed in this paper (\cref{alg:uncertain_binary}) converges. 
\begin{theorem}\label{thm:no_noise_binary}
Consider a set of $n$ \textit{i.i.d}   samples $(x_i,y_i)$ jointly sampled from $\mathcal{P}$ such that $x_i \in \rb^d$,  and  $y_i\in \{-1,1\}$ for all $i=1,\dots,n$ then under the assumption that there exists a $\theta_\star$ for which $y(\theta_\star^\top x) \geq \rho^\star$ for all $(x,y)$ pair in $\mathcal{P}$, the following convergence guarantee exists for Algorithm \ref{alg:uncertain_binary},
\begin{align}
   \E (1 - y \theta_t^\top x)_{+} \leq \frac{R^2 \max\left\{1,\frac{1}{\mu}\right\} \|\theta_1-\theta_\star\|^2}{ \min\left\{\frac{1}{\mu}, \frac{\rho^\star - 1}{1+\mu} \right\}^2 n} ,
\end{align}
for the choice of  $\sigma(\theta,x) = \frac{1}{1+\mu |\theta^\top x|}$, step size $\gamma = \frac{\min\left\{\frac{1}{\mu},\frac{\rho^\star - 1}{1+\mu}\right\}}{R^2 \max\left\{1,\frac{1}{\mu}\right\}}$ and $\|x\| \leq R $ for all $x$ in the domain~$\mathcal{X}$.
\end{theorem}

\begin{proof}[Proof sketch] See the complete proof given in the Appendix. Using the update given in \cref{eq:update_bianry} and taking expectations  only with respect to $z_t$ considering $x_t,y_t,\theta_t$ fixed, we get the following
\begin{align*}
    \E \left\|\theta_{t+1} - \theta_\star \right\|^2 &= \left\|\theta_{t} - \theta_\star \right\|^2  +2\gamma \sigma(\theta_t,x_t) \big[1 -y_t({\theta_t}^\top x_t)y_t( \theta_t^\top x_t - y_t {\theta_\star}^\top x_t )  \big]_{+} \\
    & \qquad \qquad \qquad \qquad \qquad \qquad  + \gamma^2 \sigma(\theta_t,x_t)R^2 \big[1 -y_t({\theta_t}^\top x_t) \big]_{+}^2,
\end{align*}
where $R$ is the upper bound on $\|x\|$ for all $x\in \mathcal{X}$.
Our next goal is to find the function $\sigma(\theta,x)$ which satisfy the following properties for $y_t \theta_t^\top x_t < 1$,
\begin{align}
    &\sigma(\theta_t, x_t)(1 - y_t \theta_t^\top x_t)_{+}^2 \leq c_1 (1-y_t \theta_t^\top x_t)_{+} \label{eq:sigma_bin_con1} \\
    &\sigma(\theta_t, x_t)  (y_t \theta_t^\top x_t - y_t{\theta_\star}^\top x_t) \leq -c_2 , \label{eq:sigma_bin_con2}
\end{align}
for some positive constants $c_1$ and $c_2$. We show in Lemma~\ref{lem:sampling_binary} that choosing 
\begin{align}
    \sigma(\theta,x) = \frac{1}{1+\mu |\theta^\top x|},
\end{align}
for $\mu >0$ satisfies the conditions in \cref{eq:sigma_bin_con1} and \cref{eq:sigma_bin_con2} for  $c_1 \geq \max\left\{1,\frac{1}{\mu}\right\}$ and $c_2 \leq \min\left\{\frac{\rho^\star - 1}{1+\mu},\frac{1}{\mu}\right\}$.
Finally, after  taking expectations, applying Jensen's inequality, and for the optimal choice of step size $\gamma$, we get 
\begin{align*}
    \E (1 - y \theta_t^\top x)_{+} \leq \frac{R^2 \max\left\{1,\frac{1}{\mu}\right\} \|\theta_1-\theta_\star\|^2}{ \min\left\{\frac{1}{\mu}, \frac{\rho^\star - 1}{1+\mu} \right\}^2 n}.
\end{align*}

\end{proof}
\paragraph{On Mistake Bound.} As discussed in \cref{sec:hinge},  the probability of misclassification is bounded by the expected classification loss (non squared hinge loss). Hence, for an independently sampled pair $(x,y)$,
\begin{align}
    \P(y\bar{\theta}_{n}^\top x  \leq 0) \leq \frac{R^2 \max\left\{1,\frac{1}{\mu}\right\} \|\theta_1-\theta_\star\|^2}{ \min\left\{\frac{1}{\mu}, \frac{\rho^\star - 1}{1+\mu} \right\}^2 n}.
\end{align}

\paragraph{Discussion.} It is important to note that the conditions mentioned in \cref{eq:sigma_bin_con1} and \cref{eq:sigma_bin_con2} are required only when $y_t \theta_t x_t <1$ as the gradient is zero when $y_t \theta_t x_t \geq 1$, and hence the model parameter is not updated. Ideally, the label should not be queried when $y_t \theta_t x_t \geq 1$, however there is no way to compute it beforehand. Hence, our sampling scheme provides a good trade-off for sampling. The value of $\mu$ should be decided by experimental evaluation as for $\mu\rightarrow \infty$, the upper bound for our algorithm seems to diverge. 

Denoting the expected number of samples labelled in $t$ steps as $\#_t$, we get:
\begin{align*}
    \#_n &= \sum_{t=0}^{n-1} \sigma(\theta_t, x_t) = \sum_{t=0}^{n-1} \frac{1}{1 +   \mu |\theta_t^\top x_t| },
\end{align*}
which can be significantly less than $n$ if the absolute value of $\theta_t^\top x_t$ is large for most of $t$, \textit{i.e.}, most of the points are far from the decision boundary. 
\vspace{-1mm}
\subsection{Extension to Multi-class Classification} \label{sec:multiclass}
\vspace{-1mm}

In the separable multi-class case,  we assume that there exists a set of optimal half-spaces $\theta_\star(i) \in \mathbb{R}^{d}$ for $i\in \{1,2,\ldots,k\}$ corresponding to each class such that for all $x\in \mathcal{X}$ and its corresponding label $y\in \mathcal{Y}$ where $\mathcal{Y} =  \{1,2,\ldots, k\}$,
\begin{align}
    (\theta_\star(i)-\theta_\star(j))^\top x \geq \rho^\star >1 ~\text{for all } i\neq j \in [k].
\end{align}

Note that the loss function given in \cref{eq:multiclass-hinge} is the same as that of used to compute the loss in \cref{alg:uncertain_multi-class}.  We will optimize the multi-class square hinge loss and not the hinge loss for the reason discussed in \cref{sec:hinge}. Let us introduce the following notation, 
\begin{align*}
    \delta_x(y,y') = \phi(x,y) - \phi(x,y').
\end{align*}
Hence, the expression for the square hinge loss is
\begin{align*}
  {\ell}(x,y, \theta) = \frac{1}{2}\hat{\ell}^2(x,y, \theta) = \frac{1}{2}\left[ 1 - \theta^\top \delta_{x}\left(y,y^\star(\theta,x,y)\right)\right]_+^2.
\end{align*}
The gradient of $\ell(x,y, \theta)$ with respect to $\theta$ can be written as,
\begin{align*}
    \nabla {\ell}(x,y, \theta) = -\hat{\ell}(x,y, \theta) \delta_{x}\left(y,y^\star(\theta,x,y)\right).
\end{align*}
We consider the projected stochastic gradient descent update to  to update the model parameter $\theta$ for uncertainty sampling based active learning in multi-class classification:
\begin{align}
     \theta_{t+1} = \Pi_{\|\theta\| \leq B}\left[\theta_t + \gamma z_t \delta_{x_t}\left(y_t,y^\star(\theta_t,x_t,y_t)\right)  \hat{\ell}(x_t,y_t, \theta_t)\right], \label{eq:update_multi}
\end{align}
where $z_t$ is a Bernoulli  random variable for fixed $\theta_t$ and $x_t$ such that $p(z_t=1|x_t,\theta_t) = \sigma(\theta_t, x_t)$ where $\sigma:\mathbb{R}\rightarrow [0,1]$ is an even function. $\Pi_{\|\theta\| \leq B}$ denotes the projection operator which projects $\theta_t$ for all $t$ in the ball of radius $B$ centered around origin. In the theorem below, we show the convergence of the algorithm proposed in \cref{alg:uncertain_multi-class}.

\begin{theorem}\label{thm:no_noise_multi}
Consider a set of $n$ \textit{i.i.d}   samples $(x_i,y_i)$ jointly sampled from $\mathcal{P}$ such that $x_i \in \rb^d$,  and  $y_i\in \{1,2,\ldots,k\}$ for all $i=1,\dots,n$. Then under the assumption that there exists a  set of $d$-dimensional optimal half-spaces $\theta_\star = \{\theta_\star(1),\theta_\star(2),\ldots,\theta_\star(k) \}$   corresponding to each class   for which $\theta_\star^\top \delta_{x}(y,y^\star(\theta_\star,x,y)) \geq \rho^\star$ for all $(x,y)$ pair in $\mathcal{P}$, the following convergence guarantee exists for Algorithm \ref{alg:uncertain_multi-class} under projected gradient descent update in equation \eqref{eq:update_multi},
\begin{align}
      \E \hat{\ell}(x,y,\bar{\theta}_n) \leq \frac{R^2(1+BR) \|\theta_1-\theta_\star\|^2}{ \min\left\{\frac{1}{\mu}, \frac{\rho^\star - 1}{1+\mu} \right\}^2 n},
\end{align}
for the choice of  $\sigma(\theta,x) = \frac{1}{1+\mu \left|\theta_t(y_{t(1)}^\star)^\top x_t - \theta_t(y_{t(2)}^\star)^\top x_t\right|}$, step size $\gamma = \min\left\{\frac{1}{\mu}, \frac{\rho^\star - 1}{1+\mu} \right\} \frac{1}{2R^2(1+BR)}$ and $\|\delta_x(i,j)\| \leq R $ for all $x$ in the domain $\mathcal{X}$, and for all $i \neq j \in [k]$ and $\|\theta\|\leq B$.
\end{theorem}
\begin{proof}[Proof sketch] See the complete proof  given in the Appendix. Using the update given in \cref{eq:update_multi} and taking expectations  only with respect to $z_t$ considering $x_t,y_t,\theta_t$ fixed, we get the following
\begin{align*}
    \E \|\theta_{t+1} - \theta^\star \|^2  &\leq   2\gamma \sigma(\theta_t,x_t) \hat{\ell}(x_t,y_t, \theta_t)\left( \theta_t^\top \delta_{x_t}\left(y_t,y^\star(\theta_t,x_t,y_t)\right)-  {\theta^\star}^\top \delta_{x_t}\left(y_t,y^\star(\theta_t,x_t,y_t)\right) \right) \notag \\
   &\qquad \qquad \qquad \qquad + \|\theta_{t} - \theta^\star \|^2 +  \gamma^2 \sigma(  \theta_t, x_t) R^2 \hat{\ell}^2(x_t,y_t, \theta_t),
\end{align*}
where $R$ is the upper bound on $\|\delta_x(i,j)\|$ for all $x\in \mathcal{X}$ and $i,j \in \{1,2,\dots,k\}$. Similar to the case of binary classification, we would need to find the function $\sigma(\theta,x)$ which satisfy the following properties for $\theta_t^\top \delta_{x_t}(y_t,y^\star(\theta_t,x_t,y_t)) <1$,
\begin{align}
    &\sigma(\theta_t,x_t) \hat{\ell}^2(x_t,y_t,\theta_t) \leq c_1\hat{\ell}(x_t,y_t,\theta_t) \label{eq:sigma_multi_con1}\\
    &\sigma(\theta_t,x_t) \left({\theta^\star}^\top \delta_{x_t}\left(y_t,y^\star(\theta_t,x_t,y_t)\right)    - \theta_t^\top \delta_{x_t}\left(y_t,y^\star(\theta_t,x_t,y_t)\right)   \right) \geq c_2, \label{eq:sigma_multi_con2}
\end{align}
for some positive constants $c_1$ and $c_2$. Let us now define top two predictions by the classifier $\theta$  for sample pair $(x,y)$ are as follows,
\begin{align}
    &y_{(1)}^\star = \argmax_{z\in \mathcal{Y}}\theta^\top \phi(x,z) \\
    &y_{(2)}^\star =\argmax_{z\in \mathcal{Y} \backslash y_{(1)}^\star }\theta^\top \phi(x,z).
\end{align}
Then, we show in Lemma~\ref{lem:sampling_multi} that choosing 
\begin{align*}
    \sigma(\theta_t,x_t) &= \frac{1}{1+\mu \left|\theta_t^\top \delta_{x_t}(y_{t(1)}^\star,y_{t(2)}^\star)\right|} =\frac{1}{1+\mu \left|\theta_t(y_{t(1)}^\star)^\top x_t - \theta_t(y_{t(2)}^\star)^\top x_t\right|},
\end{align*}
where $\mu$ is a positive constants, satisfies the conditions in \cref{eq:sigma_multi_con1} and \cref{eq:sigma_multi_con2} for $c_1\geq (1+BR)$ and for $c_2 \leq \frac{\rho^\star - 1}{1+\mu}$. Finally after taking expectations and applying Jensen's inequality, we get the following
\begin{align*}
    \E \hat{\ell}(x,y,\bar{\theta}_n) \leq \frac{R^2(1+BR)(1+\mu)^2\|\theta_0-\theta_\star\|^2}{n (\rho^\star - 1)^2},
\end{align*}
for all $\mu \geq 0$.
\end{proof}

\paragraph{On Mistake Bound.} Similar to the case of binary classification we  discussed in \cref{sec:hinge}, the probability of misclassification  for the multiclass case  is bounded by the expected classification loss (multi-hinge loss). Hence, for an independently sampled pair $(x,y)$ and for  $\mu\geq 0$,
\begin{align*}
    &\P(\bar{\theta}_{n}^\top \delta_{x}(y,y^\star(\bar{\theta}_{n},x,y))  \leq 0) \leq \frac{R^2(1+BR)(1+\mu)^2\|\theta_0-\theta_\star\|^2}{n (\rho^\star - 1)^2} \notag \\
    \Rightarrow~ &\P(\bar{\theta}_{n}(y)^\top x -  \bar{\theta}_{n}(y^\star(\bar{\theta}_{n},x,y))^\top x    \leq 0)    \leq \frac{R^2(1+BR)(1+\mu)^2\|\theta_0-\theta_\star\|^2}{n (\rho^\star - 1)^2}.
\end{align*}

\paragraph{Discussion.} It is directly not clear from the bound  that how to choose $\mu$ to have a direct gain of applying active learning method. However similar to the binary classification case, the querying of a label is only required when $\theta_t^\top \delta_{x_t}(y_t,y^\star(\theta_t,x_t,y_t)) <1$ as the gradient is $0$ when $\theta_t^\top \delta_{x_t}(y_t,y^\star(\theta_t,x_t,y_t)) \geq 1$. In that case, choosing larger $\mu$ will query less number of labels when $\theta_t^\top \delta_{x_t}(y_t,y^\star(\theta_t,x_t,y_t)) \geq 1$, however, then it will also start to discard informative samples. A good $\mu$ should be chosen based on experimental evidence. Nevertheless, the algorithm converges for all choices of~$\mu$. 

Denoting the expected number of samples labeled in $t$ steps as $\#_t$, we get,
\begin{align*}
    \#_n &= \sum_{t=0}^{n-1} \sigma(\theta_t,  x_t) = \sum_{t=0}^{n-1} \frac{1}{1 +  \mu \left|\theta_t^\top \delta_{x_t}(y_{t(1)}^\star,y_{t(2)}^\star)\right| },
\end{align*}
which can be significantly less than $n$ if the absolute value of $\theta_t^\top \delta_{x_t}(y_{t(1)}^\star,y_{t(2)}^\star) $ is large for most of the time instance $t$, \textit{i.e.}, most of the points are far from the decision boundary. 

\vspace{-1mm}
\subsection{Towards Uncertainty Sampling for Inseparable Data} \label{sec:insep}
\vspace{-1mm}
In this section, we  discuss   uncertainty sampling for active learning in a more realistic scenario, that is, the inseparable case. We assume the existence of mild noise in the data. That means there exists a $\theta_\star$ such that the following conditions about the classification noise hold in the the case of binary and multi-class prediction problem respectively for a given sample pair~$(x,y)$,
\begin{align}
    &\P(y \theta_\star^\top x |(x,y) \leq \rho^\star) \leq \eta ~\text{(binary classification)}, \label{eq:noise_con_bin}\\
    &\P(\theta_\star^\top \delta_{x}(y,y^\star(\theta,x,y)) |(x,y) \leq \rho^\star) \leq \eta  ~\text{(multi-class classification)} \label{eq:noise_con_multi}.
\end{align}
The assumption made about the noise in above equations are relatively stronger than  noise conditions often assumed in the statistical learning theory literature \cite{massart2006risk,tsybakov2004optimal}. However, extending our analysis under Tsyabkov's noise condition \cite{tsybakov2004optimal} is beyond the scope of this paper and could be considered in a subsequent work.  Before moving to present the main results in the noisy case, we mention below the projected stochastic gradient descent update for the binary classification, the update looks like as follows,
\begin{align}
    \theta_{t+1} = \Pi_{\| \theta\|\leq B}\left[  \theta_t  + \gamma~z_t(y_t x_t)\big[1 -y_t({\theta_t}^\top x_t) \big]_{+} \right]. \label{eq:update_bin_project}
\end{align}
Here below, we present our convergence result for inseparable data. 

\begin{theorem} \label{thm:noisy_binary}
Consider a set of $n$ \textit{i.i.d}   samples $(x_i,y_i)$ jointly sampled from $\mathcal{P}$ such that $x_i \in \rb^d$,  and  $y_i\in \{-1,1\}$ for all $i=1,\dots,n$. Then under the assumption  in equation \eqref{eq:noise_con_bin} for all $(x,y)$ pair in $\mathcal{P}$ and for the choice of $\sigma(\theta,x)= \frac{1}{1+\mu |\theta^\top x|}$, $\mu > 0$, the following convergence guarantee exists for Algorithm \ref{alg:uncertain_binary}:
\begin{enumerate}
    \item If the noise parameter  $$\eta < \frac{\min\left\{ \frac{1}{\mu} , \frac{\rho^\star - 1}{1+\mu}\right\}}{\max\left\{ R\|\theta_\star\| , \frac{1+R\|\theta_\star\|}{1+\mu} \right\} + \min\left\{ \frac{1}{\mu} , \frac{\rho^\star - 1}{1+\mu}\right\}}$$ and iterates in \cref{alg:uncertain_binary} are updated via   stochastic gradient descent update in equation~\eqref{eq:update_bianry}, then for step size $\gamma = \frac{ (1- \eta)\min\left\{\frac{1}{\mu},\frac{\rho^\star - 1}{1+\mu}\right\} -   \eta \max \left\{ \frac{1+R\|\theta_\star\|}{1+\mu},R\|\theta_\star\|\right\}}{ R^2 \max\left\{1,\frac{1}{\mu}\right\}}$, we have 
    \begin{align}
        \E (1 - y \bar{\theta}_n^\top x)_{+} \leq \frac{R^2 \max\left\{1,\frac{1}{\mu}\right\} \|\theta_1- \theta_\star \|^2}{\Gamma^2 n }\notag ,
    \end{align}
    
    such that $\|x\|\leq R$ for all $x \in \mathcal{X}$ and where \small$\Gamma = \left[(1- \eta)\min\left\{\frac{1}{\mu},\frac{\rho^\star - 1}{1+\mu}\right\} -   \eta \max \left\{ \frac{1+R\|\theta_\star\|}{1+\mu},R\|\theta_\star\|\right\} \right]$ \normalsize.
    \item If the noise parameter $\eta$ satisfies $$\eta \geq  \frac{\min\left\{ \frac{1}{\mu} , \frac{\rho^\star - 1}{1+\mu}\right\}}{\max\left\{ R\|\theta_\star\| , \frac{1+R\|\theta_\star\|}{1+\mu} \right\} + \min\left\{ \frac{1}{\mu} + \frac{\rho^\star - 1}{1+\mu}\right\}}$$ and iterates in \cref{alg:uncertain_binary} are updated via projected stochastic gradient descent update in equation~\eqref{eq:update_bin_project}, then for step size $\gamma = \frac{(1-\eta)\min\left\{\frac{1}{\mu}, \frac{\rho^\star -1}{1+\mu}\right\}}{R^2 \max\left\{1,\frac{1}{\mu}\right\}} $, we have
    \small
    \begin{align}
       \E (1 - y \bar{\theta}_n^\top x)_{+}  &\leq  \frac{\left(R^2 \max\left\{1,\frac{1}{\mu}\right\}\right) \|\theta_{1} - \theta^\star \|^2 }{ (1- \eta)^2\min\left\{\frac{1}{\mu},\frac{\rho^\star - 1}{1+\mu}\right\}^2 n}  + O(\eta) \notag ,
    \end{align}
    \normalsize
    such that $\|x\|\leq R$ for all $x \in \mathcal{X}$.
\end{enumerate}
\end{theorem}
A similar result also holds for the case of multi-class classification. 
\begin{theorem}\label{thm:noisy_multi}
Consider a set of $n$ \textit{i.i.d}   samples $(x_i,y_i)$ jointly sampled from $\mathcal{P}$ such that $x_i \in \rb^d$,  and  $y_i\in \{1,2,\ldots,k\}$ for all $i=1,\dots,n$. Then under the assumption  in equation \eqref{eq:noise_con_multi} for all $(x,y)$ pair in $\mathcal{P}$ and for the choice of $\sigma(\theta,x)= \frac{1}{1+\mu \left|\theta(y_1^\star)^\top x - \theta(y_2^\star)^\top x \right|}$, $\mu > 0$, the following convergence guarantee exists for Algorithm \ref{alg:uncertain_multi-class} with projected stochastic gradient descent update (equation~\eqref{eq:update_multi}):
\begin{enumerate}
    \item If  $\eta < \frac{\min\left\{ \frac{1}{\mu}, \frac{\rho^\star -1}{1+\mu}\right\}}{1+\min\left\{ \frac{1}{\mu}, \frac{\rho^\star -1}{1+\mu}\right\} + R\|\theta_\star\|}$, then for step size $\gamma = \frac{(1-\eta)\min\left\{\frac{1}{\mu}, \frac{\rho_\star -1}{1+\mu}\right\} - \eta (1+R\|\theta_\star\|)}{R^2(1+BR)} $
    \begin{align*}
        \E (1 - y \bar{\theta}_n^\top x)_{+} \leq \frac{R^2(1+BR) \|\theta_1 - \theta_\star\|^2}{\Gamma^2 n},
    \end{align*}
    where $\Gamma = \left[ (1-\eta)\min\left\{\frac{1}{\mu}, \frac{\rho_\star -1}{1+\mu}\right\} - \eta (1+R\|\theta_\star\|)\right]$ and  $\|\delta_x(i,j)\|\leq R$ for all $x \in \mathcal{X}$ and $i,j\in [k]$.
    \item If  $\eta \geq  \frac{\min\left\{ \frac{1}{\mu}, \frac{\rho^\star -1}{1+\mu}\right\}}{1+\min\left\{ \frac{1}{\mu}, \frac{\rho^\star -1}{1+\mu}\right\} + R\|\theta_\star\|}$, then for step size $\gamma = \frac{(1-\eta)\min\left\{\frac{1}{\mu}, \frac{\rho^\star -1}{1+\mu}\right\}}{R^2 (1+BR)} $
    \begin{align*}
       \E (1 - y \bar{\theta}_n^\top x)_{+}  &\leq  \frac{ R^2 (1+BR) \|\theta_{1} - \theta^\star \|^2 }{ (1- \eta)^2\min\left\{\frac{1}{\mu},\frac{\rho^\star - 1}{1+\mu}\right\}^2 n}  + O(\eta),
    \end{align*}
    such that $\|\delta_x(i,j)\|\leq R$ for all $x \in \mathcal{X}$ and $i,j\in [k]$.
\end{enumerate}
\end{theorem}

\paragraph{Discussion.} In the above two results, we see that our uncertainty sampling algorithm is robust to small noise. However, in the case of significantly large noise, our algorithm   has to pay for an extra error cost of the order $O(\eta)$.

\begin{figure*}[h!] 
\centering
\begin{subfigure}[t]{0.32\textwidth}
  \centering
  \includegraphics[width=\linewidth]{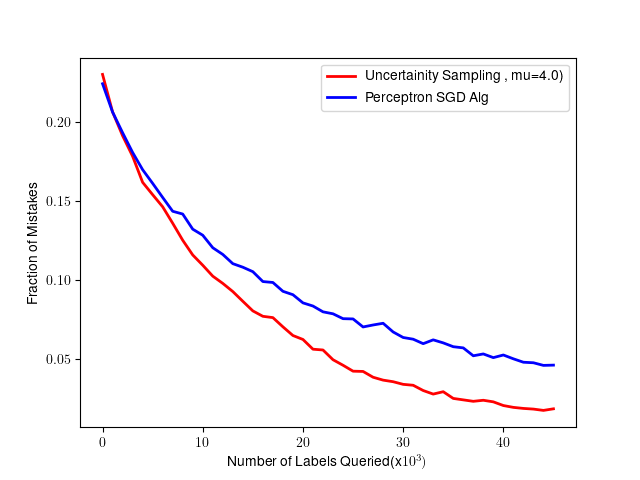}
  \caption{Test Error Comparison.}
  \label{fig:bin_syn_lin_test}
\end{subfigure}%
~
\begin{subfigure}[t]{0.32\textwidth}
  \centering
  \includegraphics[width=\linewidth]{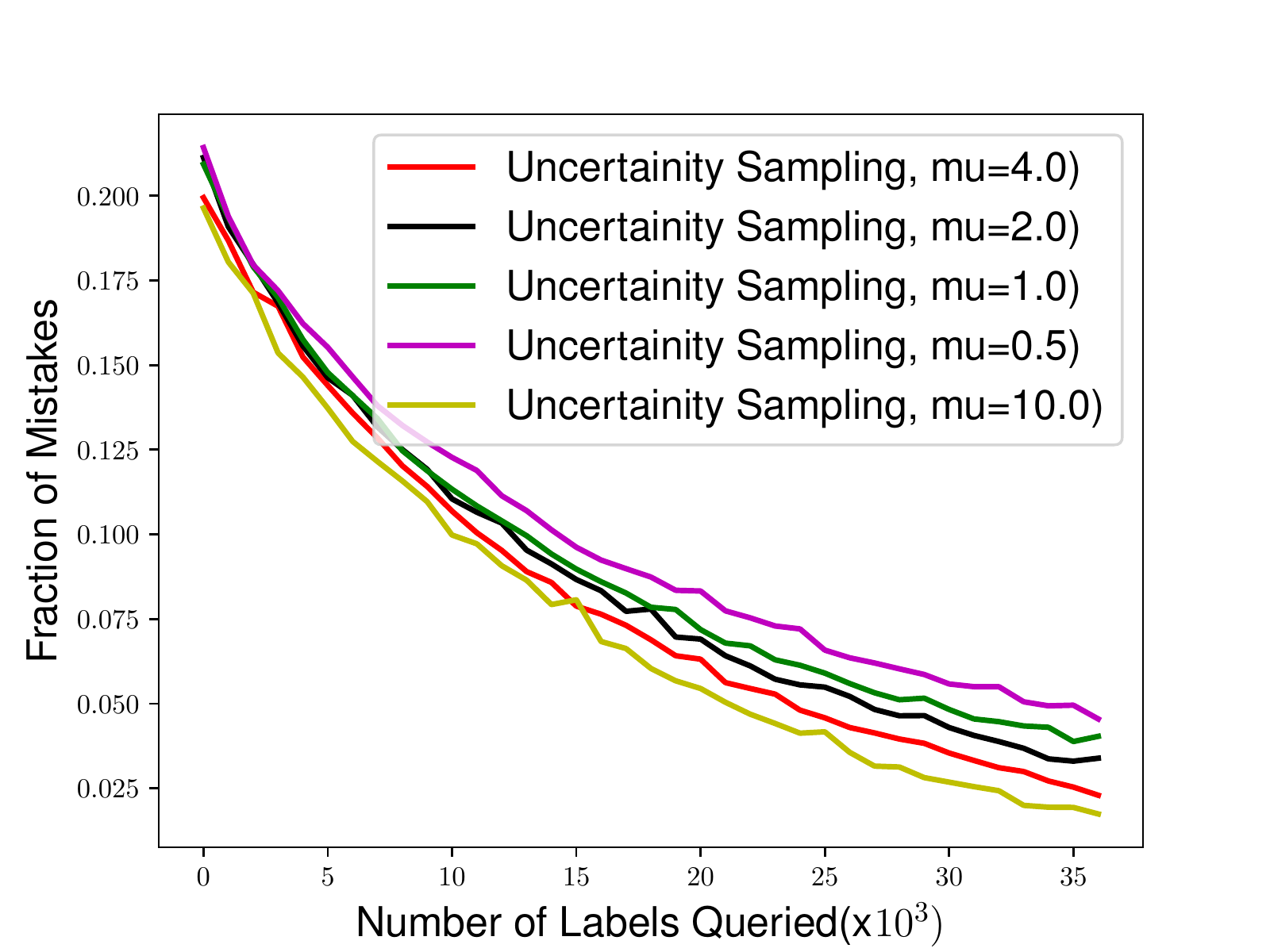}
  \caption{Comparing $\mu$.}
  \label{fig:bin_syn_lin_mu}
\end{subfigure}
~
\begin{subfigure}[t]{0.32\textwidth}
  \centering
  \includegraphics[width=\linewidth]{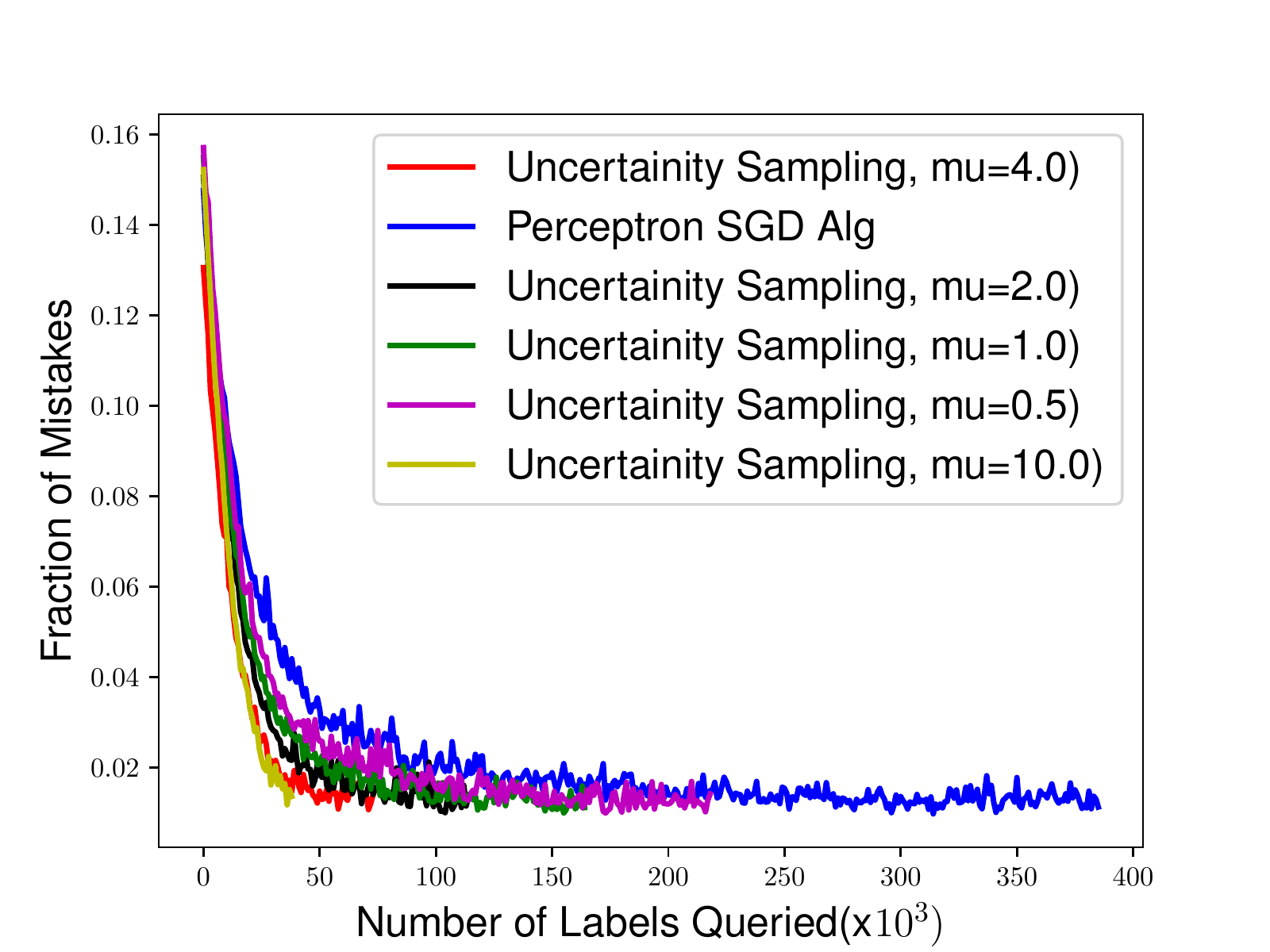}
  \caption{Unlabeled Data Requirement.}
  \label{fig:bin_syn_lin_unlabeled}
\end{subfigure}
\vspace{-2mm}
\caption{Experimental results for binary classification on synthetic data. }
 \label{fig:syn_binary}
 \vspace{-2mm}
\end{figure*}
\begin{figure*}[h!] 
\centering
\begin{subfigure}[t]{0.32\textwidth}
  \centering
  \includegraphics[width=\linewidth]{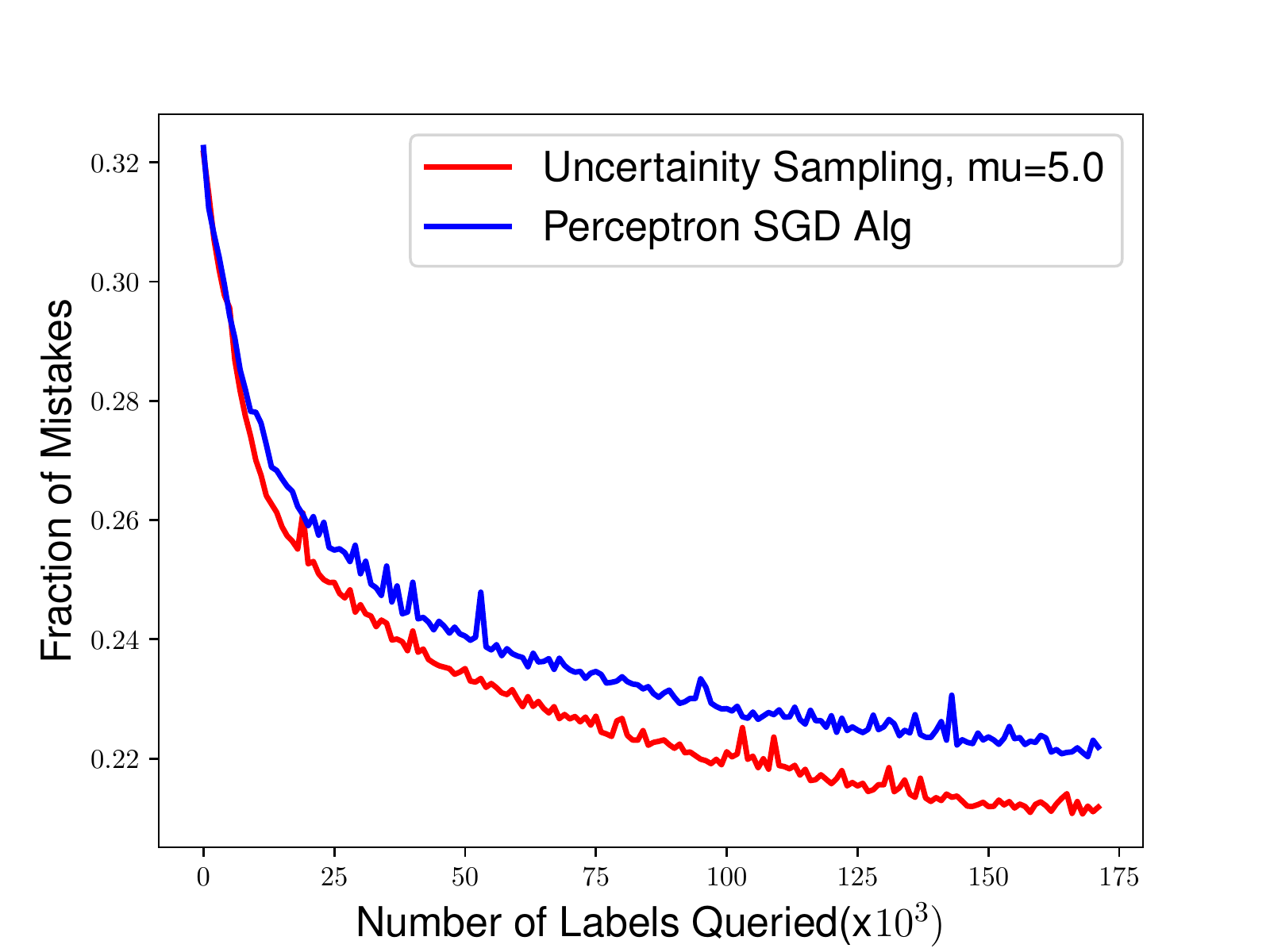}
  \caption{Test Error (Covertype dataset).}
  \label{fig:bin_cover_lin_test}
\end{subfigure}%
~
\begin{subfigure}[t]{0.32\textwidth}
  \centering
  \includegraphics[width=\linewidth]{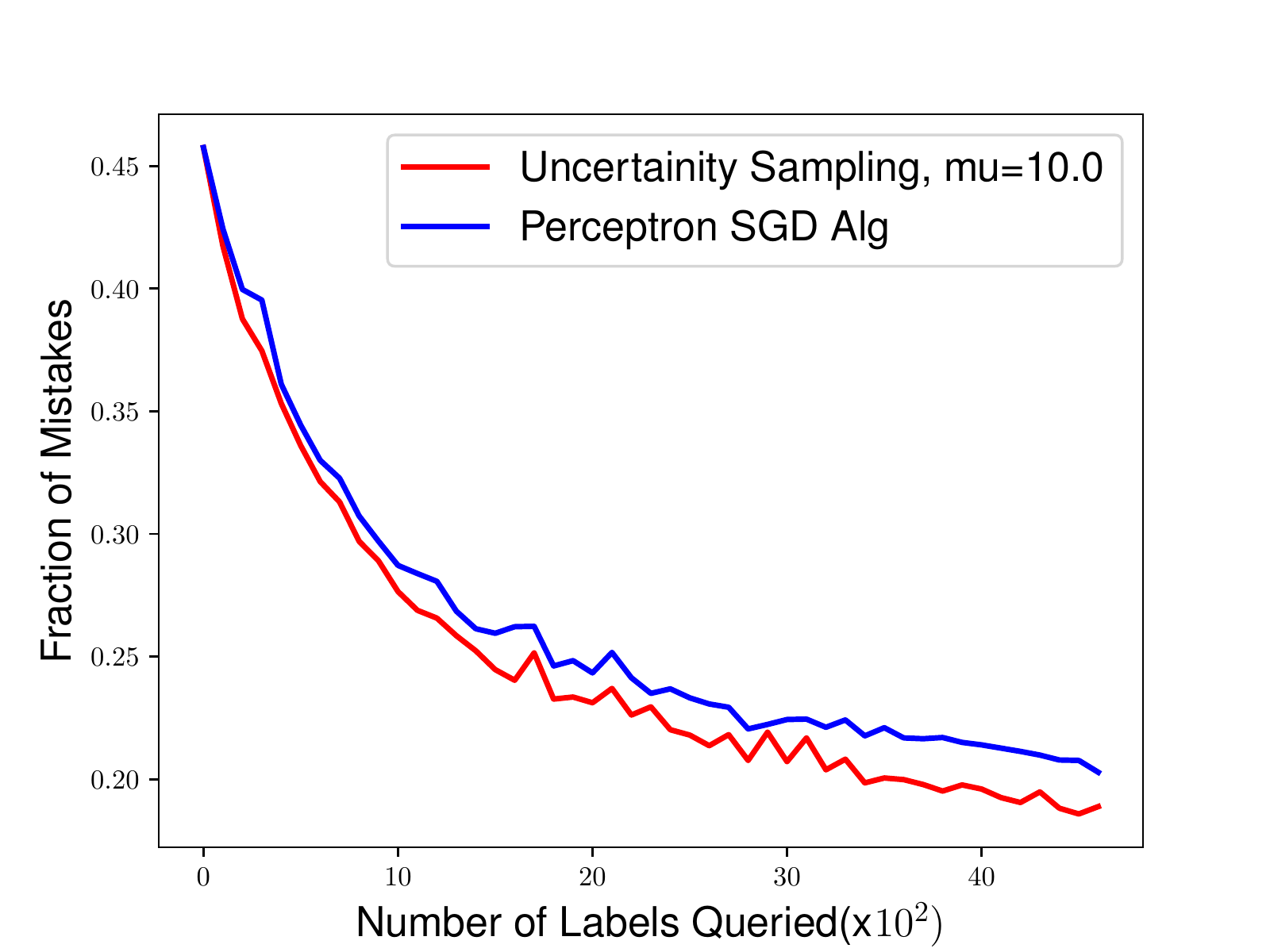}
  \caption{Test Error (letter dataset).}
  \label{fig:bin_letter_lin_mu}
\end{subfigure}
~
\begin{subfigure}[t]{0.32\textwidth}
  \centering
  \includegraphics[width=\linewidth]{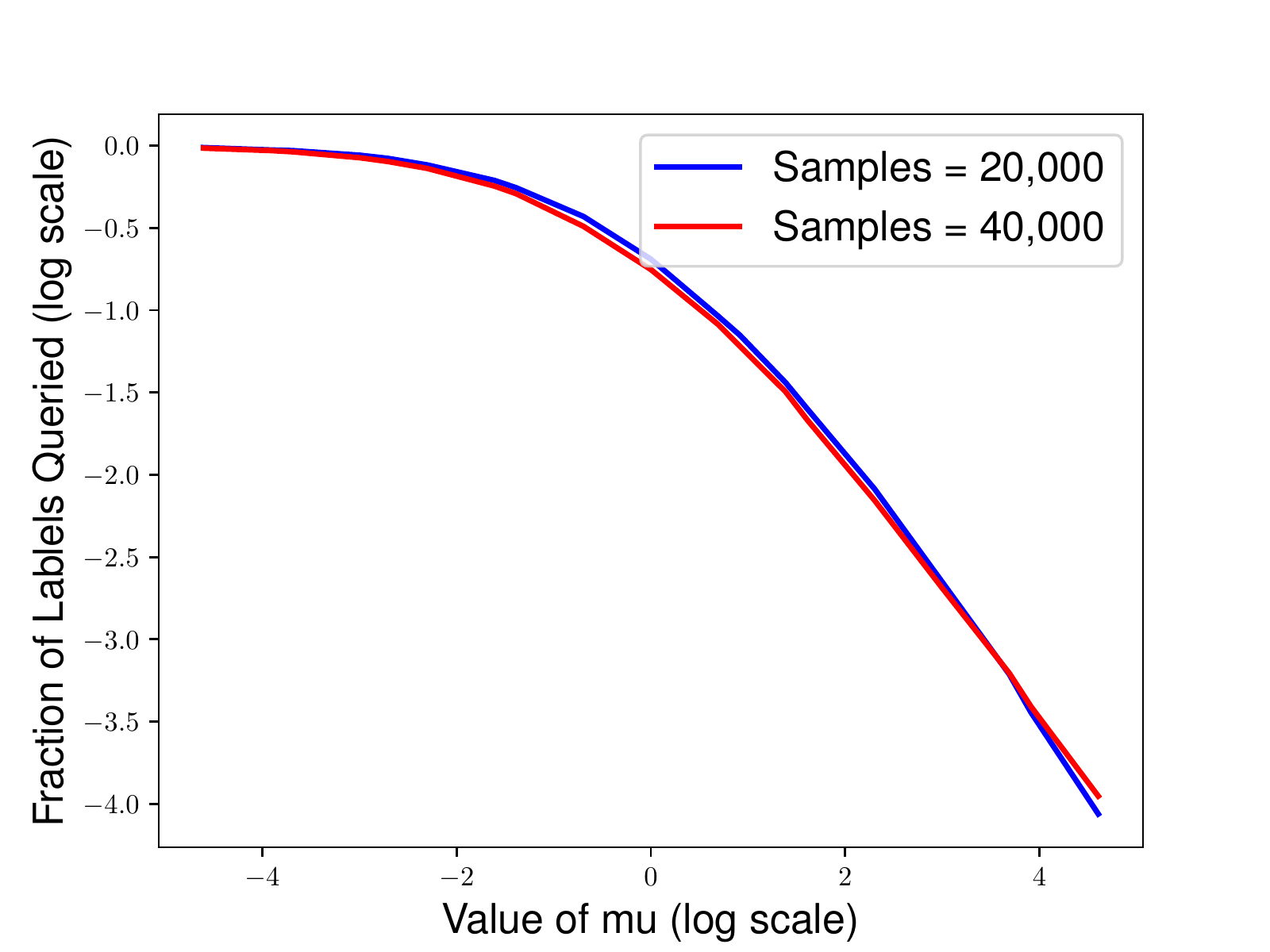}
  \caption{Number of queries vs $\mu$}
  \label{fig:evolve_mu}
\end{subfigure}
\vspace{-2mm}
\caption{Experimental results for binary classification. }
 \label{fig:real_bainary}
 \vspace{-2mm}
\end{figure*}
\vspace{-2mm}
\section{Experiments} \label{sec:expts}
\vspace{-2mm}
In this section, we perform an experimental evaluation for our proposed uncertainty sampling based active learning algorithm. The experiments are performed on both synthetic as well as real world data.  

\paragraph{Synthetic Data (Binary Classification).} We generate $n$ number of data points $x_i \in \mathbb{R}^d$ for $i \in \{1,2,\cdots,n\}$ having dimension $d=200$  from Gaussian distribution centered at 0 with covariance matrix $\Sigma$ which is a diagonal matrix. Similarly, we generate a 200-dimensional prediction vector $\theta$ sampled from the normal distribution. We compute the prediction vector $y_i$ for $x_i$ as follows, $y_i = \sign(\theta^\top x_i)$. Separately from the training set, another set of 15,000 data points are generated which are used to evaluate the test performance. 

To generate Fig.~1a, we fix $n = 200,000$, $\mu = 4$ and the $i$-th singular value of  $\Sigma$ is ${1}/{i}$. We run one pass of vanilla SGD algorithm and our uncertainty sampling based active learning algorithm. As clear from our plot, out of 200,000 samples, our algorithm asks labels for only 40,000 sample. Hence, for the comparison, we compare the performance on test set of first $40,000$ iteration of vanilla SGD with our algorithm.

In Fig.~1b, we compare the test performance for various $\mu$. It is obvious that larger value of $\mu$ implies a more aggressive sampling scheme. Hence, it is expected that for a fixed budget,  larger value of $\mu$ will provide better performance which is also reflected in the plot \ref{fig:bin_syn_lin_mu}.

In Fig.~1c, we plot the test accuracy with respect to the number of labels queries given with fixed budget of unlabeled data. As can be seen from the plot, vanilla SGD asks for the label of every data point. As we increase $\mu$, the number of queried labels goes down. However, the performance on test data remains almost the same surprisingly for uncertainty sampling based algorithm with increasing $\mu$ and for vanilla SGD at the end of one epoch.

\paragraph{Real World Datasets.} Next, we evaluate our algorithm on  real world datasets. We consider \emph{Covertype} and \emph{letter-binary} datasets to evaluate the test performance of uncertainty sampling based active learning for binary classification. Normalized binary version of datasets are downloaded from  ~\href{http://manikvarma.org/code/LDKL/download.html}{ \small \color{blue}{manikvarma.org/code/LDKL/download.html.}} \normalsize Since, the decision boundary for these datasets are non-linear,  we use randomized Fourier features \cite{rahimi2007random} to perform the task of classification. For both the datasets, we use 500 random fourier feature representation. We perform vanilla SGD and uncertainty sampling based active learning. We plot the result in Fig.~\ref{fig:bin_cover_lin_test} for \emph{covertype} and in Fig.~\ref{fig:bin_letter_lin_mu} for \emph{letter-binary}. We observe that  uncertainty sampling based active learning algorithm has low prediction error for a fixed budget. 

In the last Fig.~\ref{fig:evolve_mu}, we plot the fraction of labels queried for various value of $\mu$ when number of unlabeled synthetically generated samples are fixed to 20,000 and 40,000. The plot shows that, our algorithm tends to query similar fraction of labels irrespective of number of available unlabeled data points. 
\vspace{-2mm}
\section{Conclusion}
\vspace{-2mm}
In this paper, we show that our uncertainty sampling based active learning method converges  to the optimal predictor for linear models under the proposed sampling scheme for binary as well as for multi-class classification. We also extend our analysis for noisy case under restrictive noise assumptions (eqs.  \eqref{eq:noise_con_bin} and \eqref{eq:noise_con_multi}). As a future research direction, we would like to analyze uncertainty sampling algorithm based active learning algorithm under more generalized noise assumptions of \citet{tsybakov2004optimal}),and would like to consider more aggressive sampling schemes. 

\subsection*{Acknowledgements}
This work was funded in part by the French government under management of Agence Nationale de la Recherche as part of the “Investissements d’avenir” program, reference ANR-19-P3IA-0001 (PRAIRIE 3IA Institute). We also acknowledge support the European
Research Council (grant SEQUOIA 724063).

\bibliographystyle{plainnat}
\bibliography{opt-ml}
\normalsize
\newpage
\onecolumn
\appendix
\begin{center}
{\centering \LARGE Appendix }
\vspace{1cm}
\sloppy

\end{center}
\input{appendix}

\clearpage
\end{document}

%% file: notation.tex
\newcommand{\punt}[1]{}

\def\argmax{\mathop{\rm arg\,max}}

\def\argmax{\mathop{\rm arg\,max}}

\newcommand{\bq}{\begin{equation}}
\newcommand{\ba}{\begin{eqnarray}}
\newcommand{\ea}{\end{eqnarray}}

\newcommand{\remove}[1]{}

%% file: appendix.tex
\section{Binary Separable Classification}
\begin{theorem*}[Restatement of Theorem~\ref{thm:no_noise_binary}]
Consider a set of $n$ \textit{i.i.d}   samples $(x_i,y_i)$ jointly sampled from $\mathcal{P}$ such that $x_i \in \rb^d$,  and  $y_i\in \{-1,1\}$ for all $i=1,\dots,n$ then under the assumption that there exists a $\theta_\star$ for which $y(\theta_\star^\top x) \geq \rho^\star$ for all $(x,y)$ pair in $\mathcal{P}$, the following convergence guarantee exists for Algorithm \ref{alg:uncertain_binary},
\begin{align}
   \E (1 - y \theta_t^\top x)_{+} \leq \frac{R^2 \max\left\{1,\frac{1}{\mu}\right\} \|\theta_1-\theta_\star\|^2}{ \min\left\{\frac{1}{\mu}, \frac{\rho^\star - 1}{1+\mu} \right\}^2 n} ,
\end{align}
for the choice of  $\sigma(\theta,x) = \frac{1}{1+\mu |\theta^\top x|}$, step size $\gamma = \frac{\min\left\{\frac{1}{\mu},\frac{\rho^\star - 1}{1+\mu}\right\}}{R^2 \max\left\{1,\frac{1}{\mu}\right\}}$ and $\|x\| \leq R $ for all $x$ in the domain $\mathcal{X}$.
\end{theorem*}

\begin{proof}
We minimize the  square hinge loss which can be written as, 
\begin{align*}
    \ell(x,y,\theta)=  \frac{1}{2}\max\left[0, 1 - y\theta^\top x\right]^2.
\end{align*}
We have the following update rule for:
\begin{align*}
    \theta_{t+1} = \theta_t  + \gamma~z_t(y_t x_t)\big[1 -y_t({\theta_t}^\top x_t) \big]_{+}
\end{align*}
where $z_t$ is a bernoulli  random variable for fixed $\theta_t$ and $x_t$ such that $p(z_t=1|x_t,\theta_t) = \sigma(\theta_t, x_t)$ where $\sigma:\mathbb{R}^d\times \mathbb{R}^d\rightarrow [0,1]$ is an even function. Following the update, we have
\begin{align}
\begin{split}
    \left\|\theta_{t+1} - \theta^\star \right\|^2 &= \left\|\theta_{t} - \theta^\star + \gamma~z_t(y_tx_t)\big[1 -y_t({\theta_t}^\top x_t) \big]_{+} \right\|^2 \\
    &=\left\|\theta_{t} - \theta^\star \right\|^2  + 2\gamma z_t \big[1 -y_t({\theta_t}^\top x_t) \big]_{+}\left(y_t(\theta_t^\top x_t) - y_t({\theta^\star}^\top x_t) \right) + \gamma^2 z_t^2(y_t x_t)^2 \big[1 -y_t({\theta_t}^\top x_t) \big]_{+}^2\\
    &=\left\|\theta_{t} - \theta^\star \right\|^2  +2\gamma z_t \big[1 -y_t({\theta_t}^\top x_t) \big]_{+}\left(y_t(\theta_t^\top x_t) - y_t({\theta^\star}^\top x_t) \right) + \gamma^2 z_t(y_t x_t)^2 \big[1 -y_t({\theta_t}^\top x_t) \big]_{+}^2
\end{split}
\end{align}
In the last equation we used the fact that $z_t \in \{0,1\}$, hence $z_t^2 = z_t$. Taking expectations  on both sides only with respect to $z_t$ considering $(x_t,y_t,\theta_t)$ fixed, we have

\begin{align}
    \E \|\theta_{t+1} - \theta^\star \|^2  &=  \|\theta_{t} - \theta^\star \|^2 + 2\gamma\sigma(\theta_t, x_t) \big[1 -y_t({\theta_t}^\top x_t) \big]_{+}\left(y_t(\theta_t^\top x_t) - y_t({\theta^\star}^\top x_t) \right) \notag \\ 
    &\qquad \qquad \qquad \qquad  \qquad \qquad \qquad \qquad \qquad + \gamma^2 \sigma(  \theta_t, x_t)\|y_t x_t\|^2 \big[1 -y_t({\theta_t}^\top x_t) \big]_{+}^2 \notag \\
   &\leq \|\theta_{t} - \theta^\star \|^2 + 2\gamma \sigma(  \theta_t^\top x_t) \big[1 -y_t({\theta_t}^\top x_t) \big]_{+}\left(y_t(\theta_t^\top x_t) - y_t({\theta^\star}^\top x_t) \right) \notag \\
   &\qquad \qquad \qquad \qquad  \qquad \qquad \qquad \qquad \qquad + \gamma^2 \sigma(  \theta_t^\top x_t)R^2 \big[1 -y_t({\theta_t}^\top x_t) \big]_{+}^2.
\end{align}
Now in the above equation, we want to find function $\sigma:\mathbb{R}^d\times \mathbb{R}^d\rightarrow [0,1]$ for $y_t (\theta_t^\top x_t) <1$ for all $t$, such that the following holds,
\begin{align*}
    &\sigma(\theta_t, x_t)(1 - y_t \theta_t^\top x_t)_{+}^2 \leq c_1 (1-y_t \theta_t^\top x_t)_{+} \\
    & \sigma(\theta_t, x_t)(1 - y_t \theta_t^\top x_t)_{+} (y_t \theta_t^\top x_t - y_t{\theta^\star}^\top x_t)\leq -c_2 (1 - y_t \theta_t^\top x_t )_{+}
\end{align*}
for some positive constants $c_1$ and $c_2$. This is then equivalent to, for all $y_t \theta_t^\top x_t <1$ and $ y_t {\theta_\star}^\top x_t \geq \rho^\star$:
\begin{align}
    \sigma(\theta_t, x_t)(1 - y_t \theta_t^\top x_t)_{+} &\leq c_1 \label{eq:condition_1} \\
    \sigma(\theta_t, x_t) ( \rho^\star - y_t \theta_t^\top x_t ) &\geq c_2  \label{eq:condition_2}
\end{align}
IF we choose, $ \sigma(\theta, x) = \frac{1}{1+\mu |\theta^\top x|}$ for some constant $\mu >0$, then from \cref{lem:sampling_binary}, we have
\begin{align*}
    \frac{\min\left\{\frac{1}{\mu},\frac{\rho^\star - 1}{1+\mu}\right\}}{( \rho^\star - y \theta^\top x ) }  \leq \sigma(\theta,x)  \leq \frac{\max\left\{ 1,\frac{1}{\mu} \right\}}{(1 - y\theta^\top x)_{+}}.
\end{align*}
Hence, we have
\begin{align}
   &\E \|\theta_{t+1} - \theta^\star \|^2  \leq   \|\theta_{t} - \theta^\star \|^2 - 2\gamma \min\left\{\frac{1}{\mu},\frac{\rho^\star - 1}{1+\mu}\right\} (1 - y_t \theta_t^\top x_t)_{+} + \gamma^2 R^2 \max\left\{1,\frac{1}{\mu}\right\} (1 - y_t \theta_t^\top x_t)_{+} \notag \\
   \Rightarrow ~& (1 - y_t \theta_t^\top x_t)_{+} \leq \frac{1}{2\gamma \min\left\{\frac{1}{\mu},\frac{\rho^\star - 1}{1+\mu}\right\} - \gamma^2 R^2 \max\left\{1,\frac{1}{\mu}\right\} } \left[\|\theta_{t} - \theta^\star \|^2 - \E \|\theta_{t+1} - \theta^\star \|^2 \right] \notag 
\end{align}
Taking expectation on both the sides, we have
\begin{align}
    (1 - y \theta_t^\top x)_{+} \leq \frac{1}{2\gamma \min\left\{\frac{1}{\mu},\frac{\rho^\star - 1}{1+\mu}\right\} - \gamma^2 R^2 \max\left\{1,\frac{1}{\mu}\right\} } \left[\|\theta_{t} - \theta^\star \|^2 - \E \|\theta_{t+1} - \theta^\star \|^2 \right]
\end{align}
Summing the above equation for $t =1$ to $n$, choosing $\gamma = \frac{\min\left\{\frac{1}{\mu},\frac{\rho^\star - 1}{1+\mu}\right\}}{R^2 \max\left\{1,\frac{1}{\mu}\right\}}$ and applying Jensen's inequality give us, 
\begin{align}
    \E (1 - y \theta_t^\top x)_{+} \leq \frac{R^2 \max\left\{1,\frac{1}{\mu}\right\} \|\theta_1-\theta_\star\|^2}{ \min\left\{\frac{1}{\mu}, \frac{\rho^\star - 1}{1+\mu} \right\}^2 n}.
\end{align}

\end{proof}

\begin{lemma} \label{lem:sampling_binary}
If there exist a function $\sigma:\mathbb{R}^d\times \mathbb{R}^d\rightarrow \mathbb{R}$ of the form,
\begin{align}
    \sigma(\theta,x) = \frac{1}{1+\mu |\theta^\top x|}
\end{align}
where $\theta\in \mathbb{R}^d$ and $x\in \mathbb{R}^d$, then under the assumptions of \cref{thm:no_noise_binary} we have
\begin{align}
    \frac{\min\left\{\frac{1}{\mu},\frac{\rho^\star - 1}{1+\mu}\right\}}{( \rho^\star - y \theta^\top x ) }  \leq \sigma(\theta,x)  \leq \frac{\max\left\{ 1,\frac{1}{\mu} \right\}}{(1 - y\theta^\top x)_{+}}
\end{align}
for all $y (\theta^\top x) \leq 1$ where $y \in \{-1,1\}$.
\end{lemma}
\begin{proof}
We want to find an even function $\sigma: \mathbb{R}^d\times \mathbb{R}^d \to [0,1]$, such that for all $y (\theta^\top x) \leq 1$ where $y \in \{-1,1\}$, the following conditions hold 
\begin{align*}
    \sigma(\theta, x)(1 - y \theta^\top x)_{+}  &\leq c_1   \\
    \sigma(\theta, x)  (y \theta^\top x - \rho^\star) &\leq -c_2  
\end{align*}
for some positive constants $c_1$ and $c_2$ and  $\frac{c_1}{c_2^2}$ is as small as possible. Replacing $\sigma(\theta, x)$ with $\frac{1}{1+\mu |\theta^\top x|}$ and using the fact that $y\theta^\top x \leq 1$, we get to find constants $c_1$ and $c_2$ such that 
\begin{align*}
    c_1 &\geq \frac{1- y\theta^\top x}{1+\mu |\theta^\top x|} , \\
    c_2 &\leq \frac{\rho^\star - y\theta^\top x}{1+\mu |\theta^\top x|}.
\end{align*}
When $\mu \geq \frac{1}{\rho^\star}$, $\frac{\rho^\star - u}{1+ \mu |u|}$ is decreasing function  for $u >0$ and increasing function for $u \leq 0$. When $\mu \leq  \frac{1}{\rho^\star}$, $\frac{\rho^\star - u}{1+ \mu |u|}$ is a decreasing function everywhere. Checking at  $y\theta^\top x = 0 $, $y\theta^\top x = 1 $ and $y\theta^\top x \rightarrow -\infty$, we have
\begin{align*}
    c_1\geq \max{\left\{1,\frac{1}{\mu}\right\}} ~\text{and}~ c_2 \leq \min\left\{\frac{1}{\mu},\frac{\rho^\star - 1}{1+\mu}\right\}.
\end{align*}
Hence, finally we have
\begin{align}
   \frac{\min\left\{\frac{1}{\mu},\frac{\rho^\star - 1}{1+\mu}\right\}}{( \rho^\star - y \theta^\top x ) }  \leq \sigma(\theta,x)  \leq \frac{\max\left\{ 1,\frac{1}{\mu} \right\}}{(1 - y\theta^\top x)_{+}}.
\end{align}

\end{proof}

\section{Separable Multi-class Classification}

\begin{theorem*}[Restatement of Theorem~\ref{thm:no_noise_multi}]
Consider a set of $n$ \textit{i.i.d}   samples $(x_i,y_i)$ jointly sampled from $\mathcal{P}$ such that $x_i \in \rb^d$,  and  $y_i\in \{1,2,\ldots,k\}$ for all $i=1,\dots,n$ then under the assumption that there exists a  set of $d$-dimensional optimal half-spaces $\theta_\star = \{\theta_\star(1),\theta_\star(2),\ldots,\theta_\star(k) \}$   corresponding to each class   for which $\theta_\star^\top \delta_{x}(y,y^\star(\theta_\star,x,y)) \geq \rho^\star$ for all $(x,y)$ pair in $\mathcal{P}$, the following convergence guarantee exists for Algorithm \ref{alg:uncertain_multi-class} under projected gradient descent update in equation \eqref{eq:update_multi},
\begin{align}
      \E \hat{\ell}(x,y,\bar{\theta}_n) \leq \frac{R^2(1+BR) \|\theta_1-\theta_\star\|^2}{ \min\left\{\frac{1}{\mu}, \frac{\rho^\star - 1}{1+\mu} \right\}^2 n},
\end{align}
for the choice of  $\sigma(\theta,x) = \frac{1}{1+\mu \left|\theta_t(y_{t(1)}^\star)^\top x_t - \theta_t(y_{t(2)}^\star)^\top x_t\right|}$, step size $\gamma = \min\left\{\frac{1}{\mu}, \frac{\rho^\star - 1}{1+\mu} \right\} \frac{1}{R^2(1+BR)}$ and $\|\delta_x(i,j)\| \leq R $ for all $x$ in the domain $\mathcal{X}$, and for all $i \neq j \in [k]$ and $\|\theta\|\leq B$.
\end{theorem*}
\begin{proof}
We have multi-class hinge loss, 
\begin{align*}
    &\hat{\ell}(x_t,y_t, \theta)  = \left[ 1 - \theta^\top \delta_{x_t}\left(y_t,y^\star(\theta,x_t,y_t)\right)\right]_+,  ~\text{where}~~ y^\star(\theta,x_t,y_t) = \argmax_{y\in {\mathcal{Y}\backslash {y_t}}} \theta^\top \phi(x_t,y),
\end{align*}
and $\phi(x,y)\in \mathbb{R}^{dk}$ represents the feature map corresponding to the sample $(x,y)$. We have assumed that
\begin{align*}
    {\theta^\star}^\top \delta_{x_t}\left(y_t,y^\star(\theta,x_t,y_t)\right) \geq \rho^\star \geq 1~~\text{for all}~~t.
\end{align*}
Let us consider the smooth loss function, 
\begin{align*}
    {\ell}(x_t,y_t, \theta) = \left[\hat{\ell}(x_t,y_t, \theta)\right]^2 = \left[ 1 - \theta^\top \delta_{x_t}\left(y_t,y^\star(\theta,x_t,y_t)\right)\right]_+^2.
\end{align*}
We perform the following projected stochastic update
\begin{align*}
    \theta_{t+1} = \Pi_{\|\theta\| \leq B}\left[\theta_t + \gamma z_t \delta_{x_t}\left(y_t,y^\star(\theta_t,x_t,y_t)\right)  \left[ 1 - \theta^\top \delta_{x_t}\left(y_t,y^\star(\theta,x_t,y_t)\right)\right]_+\right]
\end{align*}
where $z_t$ is a bernoulli random variable such that $p(z_t=1|x_t,\theta_t) = \sigma(\theta_t,\top x_t)$ such that $\sigma:\mathbb{R}^{dk}\times \mathbb{R}^{dk} \rightarrow [0,1]$ is an even function. Now,

\begin{align}
\begin{split}
    \left\|\theta_{t+1} - \theta^\star \right\|^2 &\leq \left\|\theta_{t} - \theta^\star + \gamma z_t \delta_{x_t}\left(y_t,y^\star(\theta,x_t,y_t)\right)  \left[ 1 - \theta^\top \delta_{x_t}\left(y_t,y^\star(\theta_t,x_t,y_t)\right)\right]_+ \right\|^2   \\
    &=\left\|\theta_{t} - \theta^\star \right\|^2  + 2\gamma z_t \hat{\ell}(x_t,y_t, \theta_t)\left( \theta_t^\top \delta_{x_t}\left(y_t,y^\star(\theta_t,x_t,y_t)\right) -  {\theta^\star}^\top \delta_{x_t}\left(y_t,y^\star(\theta_t,x_t,y_t)\right)) \right)   \\
    & \qquad \qquad \qquad \qquad  + \gamma^2 z_t^2 \|\delta_{x_t}\left(y_t,y^\star(\theta_t,x_t,y_t)\right)\|^2 \hat{\ell}^2(x_t,y_t, \theta_t)    \\
    &=\left\|\theta_{t} - \theta^\star \right\|^2  + 2\gamma z_t \hat{\ell}(x_t,y_t, \theta_t)\left( \theta_t^\top \delta_{x_t}\left(y_t,y^\star(\theta_t,x_t,y_t)\right) -  {\theta^\star}^\top \delta_{x_t}\left(y_t,y^\star(\theta_t,x_t,y_t)\right)) \right)   \\
    & \qquad \qquad \qquad \qquad  + \gamma^2 z_t \|\delta_{x_t}\left(y_t,y^\star(\theta_t,x_t,y_t)\right)\|^2 \hat{\ell}^2(x_t,y_t, \theta_t).
\end{split}
\end{align}
In the last equation we used the fact that $z_t \in \{0,1\}$, hence $z_t^2 = z_t$. Taking expectations  on both sides only with respect to $z_t$ considering $(x_t,y_t,\theta_t)$ fixed, we have
\begin{align}
    \E \|\theta_{t+1} - \theta^\star \|^2  &\leq \left\|\theta_{t} - \theta^\star \right\|^2  + 2\gamma \sigma(\theta_t,x_t) \hat{\ell}(x_t,y_t, \theta_t)\left( \theta_t^\top \delta_{x_t}\left(y_t,y^\star(\theta_t,x_t,y_t)\right) -  {\theta^\star}^\top \delta_{x_t}\left(y_t,y^\star(\theta_t,x_t,y_t)\right)) \right) \notag   \\
    & \qquad \qquad \qquad \qquad  + \gamma^2 \sigma(\theta_t,x_t) \|\delta_{x_t}\left(y_t,y^\star(\theta_t,x_t,y_t)\right)\|^2 \hat{\ell}^2(x_t,y_t, \theta_t) \notag \\
   &\leq \|\theta_{t} - \theta^\star \|^2 + 2\gamma \sigma(\theta_t,x_t) \hat{\ell}(x_t,y_t, \theta_t)\left( \theta_t^\top \delta_{x_t}\left(y_t,y^\star(\theta_t,x_t,y_t)\right) -  {\theta^\star}^\top \delta_{x_t}\left(y_t,y^\star(\theta_t,x_t,y_t)\right)) \right) \notag \\
   &\qquad \qquad \qquad \qquad  \qquad \qquad \qquad \qquad \qquad + \gamma^2 \sigma(  \theta_t, x_t) R^2 \hat{\ell}^2(x_t,y_t, \theta_t) \notag \\
   &\leq \|\theta_{t} - \theta^\star \|^2 + 2\gamma \sigma(\theta_t,x_t) \hat{\ell}(x_t,y_t, \theta_t)\left( \theta_t^\top \delta_{x_t}\left(y_t,y^\star(\theta_t,x_t,y_t)\right) -  \rho^\star \right) \notag \\
   &\qquad \qquad \qquad \qquad  \qquad \qquad \qquad \qquad \qquad + \gamma^2 \sigma(  \theta_t, x_t) R^2 \hat{\ell}^2(x_t,y_t, \theta_t) \label{eq:multi_proof_inter}
\end{align}
Now, we want to choose such a function $\sigma$ for $\hat{\ell}(x_t,y_t, \theta_t) > 0$ such that
\begin{align*}
     & \sigma(  \theta_t, x_t) \hat{\ell}^2(x_t,y_t, \theta_t) \leq c_1  \hat{\ell}(x_t,y_t, \theta_t) \\
      & \sigma(  \theta_t, x_t) \left( \rho^\star - \theta_t^\top \delta_{x_t}\left(y_t,y^\star(\theta_t,x_t,y_t)\right)\right) \geq c_2
\end{align*}
for some positive constants $c_1$ and $c_2$. If we choose, 
$$\sigma(\theta,x) = \frac{1}{1+\mu \left|\theta_t(y_{t(1)}^\star)^\top x_t - \theta_t(y_{t(2)}^\star)^\top x_t\right|},$$
the from \cref{lem:sampling_multi} we have, 
\begin{align}
    \sigma(  \theta_t, x_t) &\leq \frac{  1+BR}{\hat{\ell}(x_t,y_t, \theta_t)},  \notag \\
    \sigma(  \theta_t, x_t) &\geq \min\left\{\frac{1}{\mu}, \frac{\rho^\star - 1}{1+\mu} \right\}\frac{1}{\left( \rho^\star - \theta_t^\top \delta_{x_t}\left(y_t,y^\star(\theta_t,x_t,y_t)\right) \right)}. \label{eq:sigma_multi_const}
\end{align}
Putting the inequalities in equation \eqref{eq:sigma_multi_const} back in the equation~\eqref{eq:multi_proof_inter}, we get the following,
\begin{align}
  & \E \|\theta_{t+1} - \theta^\star \|^2  \leq  \|\theta_{t} - \theta^\star \|^2 - 2\gamma \min\left\{\frac{1}{\mu}, \frac{\rho^\star - 1}{1+\mu} \right\} \hat{\ell}(x_t,y_t, \theta_t) + \gamma^2 R^2(1+BR)\hat{\ell}(x_t,y_t, \theta_t) \notag \\
  \Rightarrow ~& \hat{\ell}(x_t,y_t, \theta_t) \leq \frac{1}{2\gamma \min\left\{\frac{1}{\mu}, \frac{\rho^\star - 1}{1+\mu} \right\} - \gamma^2 R^2(1+BR)} \left[\|\theta_{t} - \theta^\star \|^2 - \E \|\theta_{t+1} - \theta^\star \|^2 \right] \notag 
\end{align}
Taking expectation on both the sides, we have
\begin{align}
    \hat{\ell}(x,y, \theta_t) \leq \frac{1}{2\gamma \min\left\{\frac{1}{\mu}, \frac{\rho^\star - 1}{1+\mu} \right\}  - \gamma^2 R^2(1+BR)} \left[\E \|\theta_{t} - \theta^\star \|^2 - \E \|\theta_{t+1} - \theta^\star \|^2 \right]
\end{align}
Summing the above equation for $t =1$ to $n$, choosing $\gamma = \min\left\{\frac{1}{\mu}, \frac{\rho^\star - 1}{1+\mu} \right\} \frac{1}{R^2(1+BR)}$ and applying Jensen's inequality give us, 
\begin{align}
    \E \hat{\ell}(x,y,\bar{\theta}_n) \leq \frac{R^2(1+BR) \|\theta_1-\theta_\star\|^2}{ \min\left\{\frac{1}{\mu}, \frac{\rho^\star - 1}{1+\mu} \right\}^2 n}.
\end{align}

\end{proof}

\begin{lemma} \label{lem:sampling_multi}
Consider the setting and notations of \cref{thm:no_noise_multi}. If there exist a function $\sigma:\mathbb{R}^{dk}\times \mathbb{R}^{dk}\rightarrow \mathbb{R}$ of the form,
\begin{align}
   \sigma(\theta,x) = \frac{1}{1+\mu \left|\theta(y_{(1)}^\star)^\top x - \theta(y_{(2)}^\star)^\top x\right|}
\end{align}
where $\theta\in \mathbb{R}^{dk}$ and $x\in \mathbb{R}^{dk}$, then under the assumptions of \cref{thm:no_noise_multi} we have
\begin{align}
     \min\left\{\frac{1}{\mu}, \frac{\rho^\star - 1}{1+\mu} \right\} \frac{1}{\left( \rho^\star - \theta^\top \delta_{x}\left(y,y^\star(\theta,x,y)\right) \right)} \leq  \sigma(  \theta, x) \leq \frac{  1+BR}{\hat{\ell}(x,y, \theta)}
\end{align}
for all $\theta^\top \delta_{x}\left(y,y^\star(\theta,x,y)\right) \leq 1$ where $y \in \{1,2,\ldots,k\}$.
\end{lemma}

\begin{proof}

We want to find an even function $\sigma: \mathbb{R}^d\times \mathbb{R}^d \to [0,1]$, such that for all $y (\theta^\top x) \leq 1$ where $y \in \{-1,1\}$, the following conditions hold 
\begin{align*}
    & \sigma(  \theta_t, x_t) \hat{\ell}^2(x_t,y_t, \theta_t) \leq c_1  \hat{\ell}(x_t,y_t, \theta_t) \\
      & \sigma(  \theta_t, x_t) \left( \rho^\star - \theta_t^\top \delta_{x_t}\left(y_t,y^\star(\theta_t,x_t,y_t)\right)\right) \geq c_2
\end{align*}
for some positive constants $c_1$ and $c_2$ and  $\frac{c_1}{c_2^2}$ is as small as possible. Replacing $\sigma(\theta, x)$ with $ \frac{1}{1+\mu \left|\theta(y_{(1)}^\star)^\top x - \theta(y_{(2)}^\star)^\top x\right|}$ and using the fact that $\hat{\ell}(x_t,y_t,\theta_t)y > 0$, we get to find constants $c_1$ and $c_2$ such that 
\begin{align}
    &c_1 \geq \frac{\hat{\ell}(x_t,y_t,\theta_t)}{1+\mu \left|\theta_t(y_{t(1)}^\star)^\top x_t - \theta_t(y_{t(2)}^\star)^\top x_t\right|} \\
    &c_2 \leq \frac{ \rho^\star - \theta_t^\top \delta_{x_t}\left(y_t,y^\star(\theta_t,x_t,y_t)\right)}{1+\mu \left|\theta_t(y_{t(1)}^\star)^\top x_t - \theta_t(y_{t(2)}^\star)^\top x_t\right|}
\end{align}
Under bounded $\theta$ assumption we have,
\begin{align}
     \frac{\hat{\ell}(x_t,y_t,\theta_t)}{1+\mu \left|\theta_t(y_{t(1)}^\star)^\top x_t - \theta(y_{t(2)}^\star)^\top x_t\right|} \leq 1+BR.
\end{align}
In the above equation, we have used cauchy-shwartz inequality, $\| \theta\|\leq B$ and $\| \delta_x(i,j)\| \leq R$ for all $x \in \mathcal{X}$ and $i,j \in [k]$. Let us now consider to get the inequality for $c_2$. We here observe that 
\begin{align*}
    \left|\theta_t(y_{t(1)}^\star)^\top x_t - \theta_t(y_{t(2)}^\star)^\top x_t\right| \leq \left| \theta_t^\top \delta_{x_t}\left(y_t,y^\star(\theta_t,x_t,y_t)\right) \right|
\end{align*}
Hence, 
\begin{align*}
   \frac{ \rho^\star - \theta_t^\top \delta_{x_t}\left(y_t,y^\star(\theta_t,x_t,y_t)\right)}{1+\mu \left| \theta_t^\top \delta_{x_t}\left(y_t,y^\star(\theta_t,x_t,y_t)\right) \right|} \leq \frac{ \rho^\star - \theta_t^\top \delta_{x_t}\left(y_t,y^\star(\theta_t,x_t,y_t)\right)}{1+\mu \left|\theta_t(y_{t(1)}^\star)^\top x_t - \theta_t(y_{t(2)}^\star)^\top x_t\right|}
\end{align*}
When $\mu \geq \frac{1}{\rho^\star}$, $\frac{\rho^\star - u}{1+ \mu |u|}$ is decreasing function  for $u >0$ and increasing function for $u \leq 0$. When $\mu \leq  \frac{1}{\rho^\star}$, $\frac{\rho^\star - u}{1+ \mu |u|}$ is a decreasing function everywhere. Hence, checking the value at $\theta_t^\top \delta_{x_t}\left(y_t,y^\star(\theta_t,x_t,y_t)\right) =1$ and $\theta_t^\top \delta_{x_t}\left(y_t,y^\star(\theta_t,x_t,y_t)\right) \rightarrow -\infty$ gives us the following, 
\begin{align*}
    c_2 \leq \min\left\{\frac{1}{\mu}, \frac{\rho^\star - 1}{1+\mu} \right\} \leq \frac{ \rho^\star - \theta_t^\top \delta_{x_t}\left(y_t,y^\star(\theta_t,x_t,y_t)\right)}{1+\mu \left|\theta_t(y_{t(1)}^\star)^\top x_t - \theta_t(y_{t(2)}^\star)^\top x_t\right|}.
\end{align*}
Hence, we have
\begin{align}
   \min\left\{\frac{1}{\mu}, \frac{\rho^\star - 1}{1+\mu} \right\} \frac{1}{\left( \rho^\star - \theta_t^\top \delta_{x_t}\left(y_t,y^\star(\theta_t,x_t,y_t)\right) \right)} \leq  \sigma(  \theta_t, x_t) \leq \frac{  1+BR}{\hat{\ell}(x_t,y_t, \theta_t)}.
\end{align}

\end{proof}

\section{Binary Classification in the Presence of Noise}

\begin{lemma} \label{lem:noisy_inter_bin}
Under the assumption  in equation \eqref{eq:noise_con_bin}, the iterates in \cref{alg:uncertain_binary}  updated via stochastic gradient descent (equation \eqref{eq:update_bianry}) or projected stochastic gradient descent (equation \eqref{eq:update_bin_project}) satisfy 
\begin{align}
    \E\left\|\theta_{t+1} - \theta^\star \right\|^2 &\leq \E \left\|\theta_{t} - \theta^\star \right\|^2  -2\gamma (1-\eta) \min\left\{\frac{1}{\mu},\frac{\rho^\star - 1}{1+\mu}\right\} \E (1 - y \theta_t^\top x)_{+} \notag \\
     &\qquad  + 2\gamma \eta~ \max \left\{ \frac{1+R\|\theta_\star\|}{1+\mu},R\|\theta_\star\|\right\}\E (1 - y \theta_t^\top x)_{+}   + \gamma^2 R^2 \max\left\{1,\frac{1}{\mu}\right\} \E (1 - y \theta_t^\top x)_{+}
\end{align}
for the choice of $\sigma(\theta,x)= \frac{1}{1+\mu |\theta^\top x|}$, $\mu>0$, and positive step size $\gamma$ such that for all $\|x\|\leq R$ for all $x\in \mathcal{X}$.
\end{lemma}

\begin{proof}
\begin{align}
\begin{split}
    \left\|\theta_{t+1} - \theta^\star \right\|^2 &= \left\|\theta_{t} - \theta^\star + \gamma~z_t(y_tx_t)\big[1 -y_t({\theta_t}^\top x_t) \big]_{+} \right\|^2 \\
    &=\left\|\theta_{t} - \theta^\star \right\|^2  + 2\gamma z_t \big[1 -y_t({\theta_t}^\top x_t) \big]_{+}\left(y_t(\theta_t^\top x_t) - y_t({\theta^\star}^\top x_t) \right) + \gamma^2 z_t^2(y_t x_t)^2 \big[1 -y_t({\theta_t}^\top x_t) \big]_{+}^2\\
    &=\left\|\theta_{t} - \theta^\star \right\|^2  +2\gamma z_t \big[1 -y_t({\theta_t}^\top x_t) \big]_{+}\left(y_t(\theta_t^\top x_t) - y_t({\theta^\star}^\top x_t) \right) + \gamma^2 z_t(y_t x_t)^2 \big[1 -y_t({\theta_t}^\top x_t) \big]_{+}^2
\end{split}
\end{align}
In the last equation we used the fact that $z_t \in \{0,1\}$, hence $z_t^2 = z_t$. Taking expectations  on both sides only with respect to $z_t$ considering $(x_t,y_t,\theta_t)$ fixed, we have
\begin{align}
    \E \|\theta_{t+1} - \theta^\star \|^2  &\leq  \|\theta_{t} - \theta^\star \|^2 + 2\gamma\sigma(\theta_t^\top x_t) \big[1 -y_t({\theta_t}^\top x_t) \big]_{+}\left(y_t(\theta_t^\top x_t) - y_t({\theta^\star}^\top x_t) \right) \notag \\ 
    &\qquad \qquad \qquad \qquad  \qquad \qquad \qquad \qquad \qquad + \gamma^2 \sigma(  \theta_t^\top x_t)(y_t x_t)^2 \big[1 -y_t({\theta_t}^\top x_t) \big]_{+}^2 \notag \\
   &\leq \|\theta_{t} - \theta^\star \|^2 + 2\gamma \sigma(  \theta_t^\top x_t) \big[1 -y_t({\theta_t}^\top x_t) \big]_{+}\left(y_t(\theta_t^\top x_t) - y_t({\theta^\star}^\top x_t) \right) \notag \\
   &\qquad \qquad \qquad \qquad  \qquad \qquad \qquad \qquad \qquad + \gamma^2 \sigma(  \theta_t^\top x_t)R^2 \big[1 -y_t({\theta_t}^\top x_t) \big]_{+}^2.
\end{align}
In the last line, we used the fact that $\|x\| \leq R$ for all $x \in \mathcal{X}$. Now finally we take the expectation with respect to the noise in $y_t$, using the assumption made in \cref{eq:noise_con_bin} and conditioning on $\theta_t$, $x_t$ and $y_t$. We get the following, 
\begin{align}
    \E\left\|\theta_{t+1} - \theta^\star \right\|^2 &\leq \left\|\theta_{t} - \theta^\star \right\|^2  -2\gamma (1-\eta)  \sigma(  \theta_t, x_t)  \big[1 -y_t({\theta_t}^\top x_t) \big]_{+}\left[ \left(y_t(\theta_t^\top x_t) - \rho^\star) \right) \right] \notag \\
&+ 2\gamma \eta  \sigma(  \theta_t, x_t) \big[1 -y_t({\theta_t}^\top x_t) \big]_{+}\left[  y_t(\theta_t^\top x_t) + |{\theta^\star}^\top x_t|   \right] + \gamma^2 R^2  \sigma(  \theta_t, x_t) \big[1 -y_t({\theta_t}^\top x_t) \big]_{+}^2 \notag \\
&\leq \left\|\theta_{t} - \theta^\star \right\|^2  -2\gamma (1-\eta)  \sigma(  \theta_t, x_t)  \big[1 -y_t({\theta_t}^\top x_t) \big]_{+}\left[ \left(y_t(\theta_t^\top x_t) - \rho^\star) \right) \right] \notag \\
&+ 2\gamma \eta  \sigma(  \theta_t, x_t) \big[1 -y_t({\theta_t}^\top x_t) \big]_{+}\left[  y_t(\theta_t^\top x_t) + R\|{\theta^\star}\|   \right] + \gamma^2 R^2  \sigma(  \theta_t, x_t) \big[1 -y_t({\theta_t}^\top x_t) \big]_{+}^2.
\end{align}
Similar to the noiseless case, we choose 
\begin{align*}
    \sigma(\theta, x) = \frac{1}{1+\mu |\theta^\top x|},
\end{align*}
for $\mu >0$ and after applying result from lemma~\ref{lem:sampling_binary}, we have for $y_t \theta_t^\top x_t <1$
\begin{align}
    \E\left\|\theta_{t+1} - \theta^\star \right\|^2 &\leq \left\|\theta_{t} - \theta^\star \right\|^2  -2\gamma (1-\eta) \min\left\{\frac{1}{\mu},\frac{\rho^\star - 1}{1+\mu}\right\} (1 - y_t \theta_t^\top x_t)_{+} \notag \\
   &\qquad \qquad + 2\gamma \eta \frac{y_t \theta_t^\top x_t + |\theta_\star^\top x_t|}{1+ \mu |\theta_t^\top x_t|}(1 - y_t \theta_t^\top x_t)_{+}   + \gamma^2 R^2 \max\left\{1,\frac{1}{\mu}\right\} (1 - y_t \theta_t^\top x_t)_{+} \notag \\
   &\leq \left\|\theta_{t} - \theta^\star \right\|^2  -2\gamma (1-\eta) \min\left\{\frac{1}{\mu},\frac{\rho^\star - 1}{1+\mu}\right\} (1 - y_t \theta_t^\top x_t)_{+} \notag \\
   &\qquad \qquad + 2\gamma \eta \frac{y_t \theta_t^\top x_t + R\|\theta_\star\|}{1+ \mu |\theta_t^\top x_t|}(1 - y_t \theta_t^\top x_t)_{+}   + \gamma^2 R^2 \max\left\{1,\frac{1}{\mu}\right\} (1 - y_t \theta_t^\top x_t)_{+}
\end{align}
It is easier to see that $\frac{y_t \theta_t^\top x_t + R\|\theta_\star\|}{1+ \mu |\theta_t^\top x_t|} \leq \max \left\{ \frac{1+R\|\theta_\star\|}{1+\mu},R\|\theta_\star\|\right\}$. Hence,
\begin{align}
     \E\left\|\theta_{t+1} - \theta^\star \right\|^2 &\leq \left\|\theta_{t} - \theta^\star \right\|^2  -2\gamma (1-\eta) \min\left\{\frac{1}{\mu},\frac{\rho^\star - 1}{1+\mu}\right\} (1 - y_t \theta_t^\top x_t)_{+} \notag \\
     &\qquad   + 2\gamma \eta ~\max \left\{ \frac{1+R\|\theta_\star\|}{1+\mu},R\|\theta_\star\|\right\}(1 - y_t \theta_t^\top x_t)_{+}   + \gamma^2 R^2 \max\left\{1,\frac{1}{\mu}\right\} (1 - y_t \theta_t^\top x_t)_{+}
\end{align}
Taking expectation on both sides we have, 
\begin{align}
    \E\left\|\theta_{t+1} - \theta^\star \right\|^2 &\leq \E \left\|\theta_{t} - \theta^\star \right\|^2  -2\gamma (1-\eta) \min\left\{\frac{1}{\mu},\frac{\rho^\star - 1}{1+\mu}\right\} \E (1 - y \theta_t^\top x)_{+} \notag \\
     &\qquad  + 2\gamma \eta~ \max \left\{ \frac{1+R\|\theta_\star\|}{1+\mu},R\|\theta_\star\|\right\}\E (1 - y \theta_t^\top x)_{+}   + \gamma^2 R^2 \max\left\{1,\frac{1}{\mu}\right\} \E (1 - y \theta_t^\top x)_{+}.
\end{align}

\end{proof}

\begin{theorem*}[Restatement of Theorem~\ref{thm:noisy_binary}]
Consider a set of $n$ \textit{i.i.d}   samples $(x_i,y_i)$ jointly sampled from $\mathcal{P}$ such that $x_i \in \rb^d$,  and  $y_i\in \{-1,1\}$ for all $i=1,\dots,n$ then under the assumption  in equation \eqref{eq:noise_con_bin} for all $(x,y)$ pair in $\mathcal{P}$ and for the choice of $\sigma(\theta,x)= \frac{1}{1+\mu |\theta^\top x|}$, $\mu > 0$, the following convergence guarantee exists for Algorithm \ref{alg:uncertain_binary}:
\begin{enumerate}
    \item If the noise parameter  $\eta < \frac{\min\left\{ \frac{1}{\mu} + \frac{\rho^\star - 1}{1+\mu}\right\}}{\max\left\{ R\|\theta_\star\| + \frac{1+R\|\theta_\star\|}{1+\mu} \right\} + \min\left\{ \frac{1}{\mu} + \frac{\rho^\star - 1}{1+\mu}\right\}}$ and iterates in \cref{alg:uncertain_binary} are updated via   stochastic gradient descent update in equation~\eqref{eq:update_bianry}, then for step size $\gamma = \frac{ (1- \eta)\min\left\{\frac{1}{\mu},\frac{\rho^\star - 1}{1+\mu}\right\} -   \eta \max \left\{ \frac{1+R\|\theta_\star\|}{1+\mu},R\|\theta_\star\|\right\}}{ R^2 \max\left\{1,\frac{1}{\mu}\right\}} $
    \begin{align}
        \E (1 - y \bar{\theta}_n^\top x)_{+} \leq \frac{R^2 \max\left\{1,\frac{1}{\mu}\right\} \|\theta_1- \theta_\star \|^2}{\left[(1- \eta)\min\left\{\frac{1}{\mu},\frac{\rho^\star - 1}{1+\mu}\right\} -   \eta \max \left\{ \frac{1+R\|\theta_\star\|}{1+\mu},R\|\theta_\star\|\right\} \right]^2 n },
    \end{align}
    such that $\|x\|\leq R$ for all $x \in \mathcal{X}$.
    \item If the noise parameter  $\eta \geq  \frac{\min\left\{ \frac{1}{\mu} + \frac{\rho^\star - 1}{1+\mu}\right\}}{\max\left\{ R\|\theta_\star\| + \frac{1+R\|\theta_\star\|}{1+\mu} \right\} + \min\left\{ \frac{1}{\mu} + \frac{\rho^\star - 1}{1+\mu}\right\}}$ and iterates in \cref{alg:uncertain_binary} are updated via projected stochastic gradient descent update in equation~\eqref{eq:update_bin_project}, then for step size $\gamma = \frac{(1-\eta)\min\left\{\frac{1}{\mu}, \frac{\rho^\star -1}{1+\mu}\right\}}{R^2 \max\left\{1,\frac{1}{\mu}\right\}} $
    \begin{align}
       \E (1 - y \bar{\theta}_n^\top x)_{+}  &\leq  \frac{\left(R^2 \max\left\{1,\frac{1}{\mu}\right\}\right) \|\theta_{1} - \theta^\star \|^2 }{ (1- \eta)^2\min\left\{\frac{1}{\mu},\frac{\rho^\star - 1}{1+\mu}\right\}^2 n}  + O(\eta),
    \end{align}
    such that $\|x\|\leq R$ for all $x \in \mathcal{X}$.
\end{enumerate}
\end{theorem*}
\begin{proof}
From \cref{lem:noisy_inter_bin}, we have
\begin{align*}
   \E\left\|\theta_{t+1} - \theta^\star \right\|^2 &\leq \E \left\|\theta_{t} - \theta^\star \right\|^2  -2\gamma (1-\eta) \min\left\{\frac{1}{\mu},\frac{\rho^\star - 1}{1+\mu}\right\} \E (1 - y \theta_t^\top x)_{+} \notag \\
     &\qquad  + 2\gamma \eta~ \max \left\{ \frac{1+R\|\theta_\star\|}{1+\mu},R\|\theta_\star\|\right\}\E (1 - y \theta_t^\top x)_{+}   + \gamma^2 R^2 \max\left\{1,\frac{1}{\mu}\right\} \E (1 - y \theta_t^\top x)_{+} 
\end{align*}
Now we consider the two cases.
\begin{enumerate}
    \item[i] When $$\eta < \frac{\min\left\{ \frac{1}{\mu} + \frac{\rho^\star - 1}{1+\mu}\right\}}{\max\left\{ R\|\theta_\star\| + \frac{1+R\|\theta_\star\|}{1+\mu} \right\} + \min\left\{ \frac{1}{\mu} + \frac{\rho^\star - 1}{1+\mu}\right\}},$$
    then, $(1-\eta) \min\left\{\frac{1}{\mu},\frac{\rho^\star - 1}{1+\mu}\right\} > \eta~ \max \left\{ \frac{1+R\|\theta_\star\|}{1+\mu},R\|\theta_\star\|\right\}$. Hence, 
    \begin{align*}
        &\E\left\|\theta_{t+1} - \theta^\star \right\|^2 \leq \E \left\|\theta_{t} - \theta^\star \right\|^2  -2\gamma (1-\eta) \min\left\{\frac{1}{\mu},\frac{\rho^\star - 1}{1+\mu}\right\} \E (1 - y \theta_t^\top x)_{+} \notag \\
     &\qquad \qquad \qquad  + 2\gamma \eta~ \max \left\{ \frac{1+R\|\theta_\star\|}{1+\mu},R\|\theta_\star\|\right\}\E (1 - y \theta_t^\top x)_{+}   + \gamma^2 R^2 \max\left\{1,\frac{1}{\mu}\right\} \E (1 - y \theta_t^\top x)_{+} \\
     \Rightarrow~& \E (1 - y \theta_t^\top x)_{+} \leq \frac{\E \|\theta_{t} - \theta^\star \|^2 - \E \|\theta_{t+1} - \theta^\star \|^2}{2\gamma (1- \eta)\min\left\{\frac{1}{\mu},\frac{\rho^\star - 1}{1+\mu}\right\} - 2\gamma \eta \max \left\{ \frac{1+R\|\theta_\star\|}{1+\mu},R\|\theta_\star\|\right\} - \gamma^2 R^2 \max\left\{1,\frac{1}{\mu}\right\} } 
    \end{align*}
    Now in the above equation, summing for all $t$ from 1 to $n$ and applying Jensen's inequality after choosing the optimal step size $\gamma = \frac{ (1- \eta)\min\left\{\frac{1}{\mu},\frac{\rho^\star - 1}{1+\mu}\right\} -   \eta \max \left\{ \frac{1+R\|\theta_\star\|}{1+\mu},R\|\theta_\star\|\right\}}{ R^2 \max\left\{1,\frac{1}{\mu}\right\}}$, we get 
    \begin{align}
        \E (1 - y \bar{\theta}_n^\top x)_{+} \leq \frac{R^2 \max\left\{1,\frac{1}{\mu}\right\} \|\theta_1- \theta_\star \|^2}{\left[(1- \eta)\min\left\{\frac{1}{\mu},\frac{\rho^\star - 1}{1+\mu}\right\} -   \eta \max \left\{ \frac{1+R\|\theta_\star\|}{1+\mu},R\|\theta_\star\|\right\} \right]^2 n }.
    \end{align}

    \item[ii] When $$\eta \geq  \frac{\min\left\{ \frac{1}{\mu} + \frac{\rho^\star - 1}{1+\mu}\right\}}{\max\left\{ R\|\theta_\star\| + \frac{1+R\|\theta_\star\|}{1+\mu} \right\} + \min\left\{ \frac{1}{\mu} + \frac{\rho^\star - 1}{1+\mu}\right\}},$$ 
    then $(1-\eta) \min\left\{\frac{1}{\mu},\frac{\rho^\star - 1}{1+\mu}\right\} < \eta~ \max \left\{ \frac{1+R\|\theta_\star\|}{1+\mu},R\|\theta_\star\|\right\}$. Hence,
     \begin{align*}
        &\E\left\|\theta_{t+1} - \theta^\star \right\|^2 \leq \E \left\|\theta_{t} - \theta^\star \right\|^2  -2\gamma (1-\eta) \min\left\{\frac{1}{\mu},\frac{\rho^\star - 1}{1+\mu}\right\} \E (1 - y \theta_t^\top x)_{+} \notag \\
     &\qquad \qquad \qquad  + 2\gamma \eta~ \max \left\{ \frac{1+R\|\theta_\star\|}{1+\mu},R\|\theta_\star\|\right\}\E (1 - y \theta_t^\top x)_{+}   + \gamma^2 R^2 \max\left\{1,\frac{1}{\mu}\right\} \E (1 - y \theta_t^\top x)_{+} \\
     \Rightarrow~& \E (1 - y \theta_t^\top x)_{+} \leq \frac{\|\theta_{t} - \theta^\star \|^2 - \E \|\theta_{t+1} - \theta^\star \|^2}{2\gamma (1- \eta)\min\left\{\frac{1}{\mu},\frac{\rho^\star - 1}{1+\mu}\right\}  - \gamma^2 R^2 \max\left\{1,\frac{1}{\mu}\right\} } \\
     &\qquad \qquad \qquad \qquad \qquad + \frac{2\gamma \eta \max \left\{ \frac{1+R\|\theta_\star\|}{1+\mu},R\|\theta_\star\|\right\}}{2\gamma (1- \eta)\min\left\{\frac{1}{\mu},\frac{\rho^\star - 1}{1+\mu}\right\}  - \gamma^2 R^2 \max\left\{1,\frac{1}{\mu}\right\}} \E (1 - y \theta_t^\top x)_{+}
    \end{align*}
    We update via projected stochastic gradient descent. Hence,
    $$(1 - y\theta_t^\top x)_+ \leq 1+ BR$$
    for all $(x,y)$ pair. This gives,
    \begin{align*}
        &\E (1 - y \theta_t^\top x)_{+} \leq \frac{\E\|\theta_{t} - \theta^\star \|^2 - \E \|\theta_{t+1} - \theta^\star \|^2}{2\gamma (1- \eta)\min\left\{\frac{1}{\mu},\frac{\rho^\star - 1}{1+\mu}\right\}  - \gamma^2 R^2 \max\left\{1,\frac{1}{\mu}\right\} } \\
     &\qquad \qquad  \qquad  \qquad + \frac{2\gamma \eta \max \left\{ \frac{1+R\|\theta_\star\|}{1+\mu},R\|\theta_\star\|\right\}(1+BR)}{2\gamma (1- \eta)\min\left\{\frac{1}{\mu},\frac{\rho^\star - 1}{1+\mu}\right\}  - \gamma^2 R^2 \max\left\{1,\frac{1}{\mu}\right\}}.
    \end{align*}
    Choosing optimal $\gamma = \frac{(1-\eta)\min\left\{\frac{1}{\mu}, \frac{\rho^\star -1}{1+\mu}\right\}}{R^2 \max\left\{1,\frac{1}{\mu}\right\}}$ we have
    \begin{align*}
    \E (1 - y \theta_t^\top x)_{+}  &\leq \frac{\left(R^2 \max\left\{1,\frac{1}{\mu}\right\}\right)\left[\E\|\theta_{t} - \theta^\star \|^2 - \E \|\theta_{t+1} - \theta^\star \|^2\right]}{ (1- \eta)^2\min\left\{\frac{1}{\mu},\frac{\rho^\star - 1}{1+\mu}\right\}^2} \\
    &\qquad \qquad \qquad + {\left(\frac{2(1+BR) \max\left\{\frac{1+R\|\theta_\star\|}{1+\mu},R\|\theta_\star\| \right\} }{(1-\eta)\min\left\{\frac{1}{\mu}, \frac{\rho^\star -1}{1+\mu}\right\}}\right)} ~\eta .
    \end{align*} 
    Summing the above equation for $t =1$ to $n$ and applying Jensen's inequality give us, 
    \begin{align}
        \E (1 - y \bar{\theta}_n^\top x)_{+}  &\leq \frac{\left(R^2 \max\left\{1,\frac{1}{\mu}\right\}\right) \|\theta_{1} - \theta^\star \|^2 }{ (1- \eta)^2\min\left\{\frac{1}{\mu},\frac{\rho^\star - 1}{1+\mu}\right\}^2 n}  + \underbrace{\left(\frac{2(1+BR) \max\left\{\frac{1+R\|\theta_\star\|}{1+\mu},R\|\theta_\star\| \right\} }{(1-\eta)\min\left\{\frac{1}{\mu}, \frac{\rho^\star -1}{1+\mu}\right\}}\right)}_{:\approx \text{constant}}~\eta \\
    &= \frac{\left(R^2 \max\left\{1,\frac{1}{\mu}\right\}\right) \|\theta_{1} - \theta^\star \|^2 }{ (1- \eta)^2\min\left\{\frac{1}{\mu},\frac{\rho^\star - 1}{1+\mu}\right\}^2 n}  + O(\eta).
    \end{align}
\end{enumerate}
\end{proof}

\section{Multi-class Classification in the Presence of Noise}

\begin{lemma} \label{lem:noisy_inter_multi}
Under the assumption  in equation \eqref{eq:noise_con_multi}, the iterates in \cref{alg:uncertain_multi-class}  updated via  projected stochastic gradient descent (equation \eqref{eq:update_multi}) satisfy 
\begin{align}
     \E \|\theta_{t+1} - \theta^\star \|^2 &\leq \E \|\theta_{t} - \theta^\star \|^2 - 2\gamma (1-\eta) \min\left\{\frac{1}{\mu},\frac{\rho^\star - 1}{1+\mu}\right\} \E[\hat{\ell}(x,y, \theta_t)] \notag \\
   &\qquad  +2 \gamma \eta(1+R\|\theta_\star\|)  \E[\hat{\ell}(x,y, \theta_t)]   + \gamma^2 (1+BR) R^2 \E[\hat{\ell}(x,y, \theta_t)]
\end{align}
for the choice of $\sigma(\theta,x)= \frac{1}{1+\mu \left|\theta(y_1^\star)^\top x - \theta(y_2^\star)^\top x \right|}$, $\mu>0$, and positive step size $\gamma$ such that for all $\|\delta_x(i,j)\|\leq R$ for all $x\in \mathcal{X}$ and $i,j \in [k]$.
\end{lemma}

\begin{proof}
We perform the following projected stochastic update
\begin{align*}
    \theta_{t+1} = \Pi_{\|\theta\| \leq B}\left[\theta_t + \gamma z_t \delta_{x_t}\left(y_t,y^\star(\theta_t,x_t,y_t)\right)  \left[ 1 - \theta^\top \delta_{x_t}\left(y_t,y^\star(\theta,x_t,y_t)\right)\right]_+\right]
\end{align*}
where $z_t$ is a bernoulli random variable such that $p(z_t=1|x_t,\theta_t) = \sigma(\theta_t,\top x_t)$ such that $\sigma:\mathbb{R}^{dk}\times \mathbb{R}^{dk} \rightarrow [0,1]$ is an even function. Now,

\begin{align}
\begin{split}
    \left\|\theta_{t+1} - \theta^\star \right\|^2 &\leq \left\|\theta_{t} - \theta^\star + \gamma z_t \delta_{x_t}\left(y_t,y^\star(\theta,x_t,y_t)\right)  \left[ 1 - \theta^\top \delta_{x_t}\left(y_t,y^\star(\theta_t,x_t,y_t)\right)\right]_+ \right\|^2   \\
    &=\left\|\theta_{t} - \theta^\star \right\|^2  + 2\gamma z_t \hat{\ell}(x_t,y_t, \theta_t)\left( \theta_t^\top \delta_{x_t}\left(y_t,y^\star(\theta_t,x_t,y_t)\right) -  {\theta^\star}^\top \delta_{x_t}\left(y_t,y^\star(\theta_t,x_t,y_t)\right)) \right)   \\
    & \qquad \qquad \qquad \qquad  + \gamma^2 z_t^2 \|\delta_{x_t}\left(y_t,y^\star(\theta_t,x_t,y_t)\right)\|^2 \hat{\ell}^2(x_t,y_t, \theta_t)    \\
    &=\left\|\theta_{t} - \theta^\star \right\|^2  + 2\gamma z_t \hat{\ell}(x_t,y_t, \theta_t)\left( \theta_t^\top \delta_{x_t}\left(y_t,y^\star(\theta_t,x_t,y_t)\right) -  {\theta^\star}^\top \delta_{x_t}\left(y_t,y^\star(\theta_t,x_t,y_t)\right)) \right)   \\
    & \qquad \qquad \qquad \qquad  + \gamma^2 z_t \|\delta_{x_t}\left(y_t,y^\star(\theta_t,x_t,y_t)\right)\|^2 \hat{\ell}^2(x_t,y_t, \theta_t).
\end{split}
\end{align}
In the last equation we used the fact that $z_t \in \{0,1\}$, hence $z_t^2 = z_t$.  Taking expectations  on both sides only with respect to $z_t$ considering $(x_t,y_t,\theta_t)$ fixed, we have
\begin{align}
    \E \|\theta_{t+1} - \theta^\star \|^2  &\leq \left\|\theta_{t} - \theta^\star \right\|^2  + 2\gamma \sigma(\theta_t,x_t) \hat{\ell}(x_t,y_t, \theta_t)\left( \theta_t^\top \delta_{x_t}\left(y_t,y^\star(\theta_t,x_t,y_t)\right) -  {\theta^\star}^\top \delta_{x_t}\left(y_t,y^\star(\theta_t,x_t,y_t)\right)) \right) \notag   \\
    & \qquad \qquad \qquad \qquad  + \gamma^2 \sigma(\theta_t,x_t) \|\delta_{x_t}\left(y_t,y^\star(\theta_t,x_t,y_t)\right)\|^2 \hat{\ell}^2(x_t,y_t, \theta_t) \notag \\
   &\leq \|\theta_{t} - \theta^\star \|^2 + 2\gamma \sigma(\theta_t,x_t) \hat{\ell}(x_t,y_t, \theta_t)\left( \theta_t^\top \delta_{x_t}\left(y_t,y^\star(\theta_t,x_t,y_t)\right) -  {\theta^\star}^\top \delta_{x_t}\left(y_t,y^\star(\theta_t,x_t,y_t)\right)) \right) \notag \\
   &\qquad \qquad \qquad \qquad  \qquad \qquad \qquad \qquad \qquad + \gamma^2 \sigma(  \theta_t, x_t) R^2 \hat{\ell}^2(x_t,y_t, \theta_t)  \label{eq:multi_proof_inter_noise}
\end{align}
In the last line, we used the fact that $\|\delta_x(i,j)\| \leq R$ for all $x \in \mathcal{X}$ and $i,j \in [k]$. Now finally we take the expectation with respect to the noise in $y_t$, using the assumption made in \cref{eq:noise_con_multi} and conditioning on $\theta_t$, $x_t$ and $y_t$. We get the following, 
\begin{align*}
    \E \|\theta_{t+1} - \theta^\star \|^2  &\leq \|\theta_{t} - \theta^\star \|^2 + 2\gamma (1-\eta) \sigma(\theta_t,x_t) \hat{\ell}(x_t,y_t, \theta_t)\left( \theta_t^\top \delta_{x_t}\left(y_t,y^\star(\theta_t,x_t,y_t)\right) -   \rho^\star \right) \notag \\
   &\qquad  +2 \gamma \eta \sigma(\theta_t,x_t) \hat{\ell}(x_t,y_t, \theta_t)\left( \theta_t^\top \delta_{x_t}\left(y_t,y^\star(\theta_t,x_t,y_t)\right) + \left| \theta_\star^\top \delta_{x_t}\left(y_t,y^\star(\theta_t,x_t,y_t)\right) \right| \right) \\
   &\qquad \qquad \qquad \qquad \qquad \qquad \qquad \qquad \qquad  + \gamma^2 \sigma(  \theta_t, x_t) R^2 \hat{\ell}^2(x_t,y_t, \theta_t) \\
    &\leq \|\theta_{t} - \theta^\star \|^2 + 2\gamma (1-\eta) \sigma(\theta_t,x_t) \hat{\ell}(x_t,y_t, \theta_t)\left( \theta_t^\top \delta_{x_t}\left(y_t,y^\star(\theta_t,x_t,y_t)\right) -   \rho^\star \right) \notag \\
   &\qquad  +2 \gamma \eta \sigma(\theta_t,x_t) \hat{\ell}(x_t,y_t, \theta_t)\left( \theta_t^\top \delta_{x_t}\left(y_t,y^\star(\theta_t,x_t,y_t)\right) +R\|\theta_\star\| \right) \\
   &\qquad \qquad \qquad \qquad \qquad \qquad \qquad \qquad \qquad + \gamma^2 \sigma(  \theta_t, x_t) R^2 \hat{\ell}^2(x_t,y_t, \theta_t)
\end{align*}
Similar to the noiseless case, 
we choose $$\sigma(\theta, x) = \frac{1}{1+\mu \left|\theta(y_1^\star)^\top x - \theta(y_2^\star)^\top x \right|}$$
for $\mu >0$ and after applying result from lemma~\ref{lem:sampling_multi}, we have for $\hat{\ell}^2(x_t,y_t, \theta_t) > 0$
\begin{align}
    \E \|\theta_{t+1} - \theta^\star \|^2 &\leq \|\theta_{t} - \theta^\star \|^2 - 2\gamma (1-\eta) \min\left\{\frac{1}{\mu},\frac{\rho^\star - 1}{1+\mu}\right\} \hat{\ell}(x_t,y_t, \theta_t) \notag \\
   &\qquad  +2 \gamma \eta~   \hat{\ell}(x_t,y_t, \theta_t)\frac{\left( \theta_t^\top \delta_{x_t}\left(y_t,y^\star(\theta_t,x_t,y_t)\right) +R\|\theta_\star\| \right)}{1+ \mu \left|\theta_t(y_{t(1)}^\star)^\top x_t - \theta_t(y_{t(2)}^\star)^\top x_t\right| }\notag  \\
   &\qquad \qquad \qquad \qquad \qquad \qquad \qquad   + \gamma^2 (1+BR) R^2 \hat{\ell}(x_t,y_t, \theta_t) \notag \\
   &\leq \|\theta_{t} - \theta^\star \|^2 - 2\gamma (1-\eta) \min\left\{\frac{1}{\mu},\frac{\rho^\star - 1}{1+\mu}\right\} \hat{\ell}(x_t,y_t, \theta_t) \notag \\
   &\qquad  +2 \gamma \eta(1+R\|\theta_\star\|)  \hat{\ell}(x_t,y_t, \theta_t)   + \gamma^2 (1+BR) R^2 \hat{\ell}(x_t,y_t, \theta_t).
\end{align}
Taking expectation on both sides of the above expression gives us
\begin{align}
    \E \|\theta_{t+1} - \theta^\star \|^2 &\leq \E \|\theta_{t} - \theta^\star \|^2 - 2\gamma (1-\eta) \min\left\{\frac{1}{\mu},\frac{\rho^\star - 1}{1+\mu}\right\} \E[\hat{\ell}(x,y, \theta_t)] \notag \\
   &\qquad  +2 \gamma \eta(1+R\|\theta_\star\|)  \E[\hat{\ell}(x,y, \theta_t)]   + \gamma^2 (1+BR) R^2 \E[\hat{\ell}(x,y, \theta_t)]
\end{align}
\end{proof}

\begin{theorem*}[Restatement of Theorem~\ref{thm:noisy_multi}]
Consider a set of $n$ \textit{i.i.d}   samples $(x_i,y_i)$ jointly sampled from $\mathcal{P}$ such that $x_i \in \rb^d$,  and  $y_i\in \{1,2,\ldots,k\}$ for all $i=1,\dots,n$ then under the assumption  in equation \eqref{eq:noise_con_multi} for all $(x,y)$ pair in $\mathcal{P}$ and for the choice of $\sigma(\theta,x)= \frac{1}{1+\mu \left|\theta(y_1^\star)^\top x - \theta(y_2^\star)^\top x \right|}$, $\mu > 0$, the following convergence guarantee exists for Algorithm \ref{alg:uncertain_multi-class} with projected stochastic gradient descent update (equation~\eqref{eq:update_multi}):
\begin{enumerate}
    \item If the noise parameter  $\eta < \frac{\min\left\{ \frac{1}{\mu}, \frac{\rho^\star -1}{1+\mu}\right\}}{1+\min\left\{ \frac{1}{\mu}, \frac{\rho^\star -1}{1+\mu}\right\} + R\|\theta_\star\|}$, then for step size $\gamma = \frac{(1-\eta)\min\left\{\frac{1}{\mu}, \frac{\rho_\star -1}{1+\mu}\right\} - \eta (1+R\|\theta_\star\|)}{R^2(1+BR)} $
    \begin{align}
        \E (1 - y \bar{\theta}_n^\top x)_{+} \leq \frac{R^2(1+BR) \|\theta_1 - \theta_\star\|^2}{\left[ (1-\eta)\min\left\{\frac{1}{\mu}, \frac{\rho_\star -1}{1+\mu}\right\} - \eta (1+R\|\theta_\star\|)\right]^2n},
    \end{align}
    such that $\|\delta_x(i,j)\|\leq R$ for all $x \in \mathcal{X}$ and $i,j\in [k]$.
    \item If the noise parameter  $\eta \geq  \frac{\min\left\{ \frac{1}{\mu}, \frac{\rho^\star -1}{1+\mu}\right\}}{1+\min\left\{ \frac{1}{\mu}, \frac{\rho^\star -1}{1+\mu}\right\} + R\|\theta_\star\|}$, then for step size $\gamma = \frac{(1-\eta)\min\left\{\frac{1}{\mu}, \frac{\rho^\star -1}{1+\mu}\right\}}{R^2 (1+BR)} $
    \begin{align}
       \E (1 - y \bar{\theta}_n^\top x)_{+}  &\leq  \frac{ R^2 (1+BR) \|\theta_{1} - \theta^\star \|^2 }{ (1- \eta)^2\min\left\{\frac{1}{\mu},\frac{\rho^\star - 1}{1+\mu}\right\}^2 n}  + O(\eta),
    \end{align}
    such that $\|\delta_x(i,j)\|\leq R$ for all $x \in \mathcal{X}$ and $i,j\in [k]$.
\end{enumerate}
\end{theorem*}
\begin{proof}
From lemma~\ref{lem:noisy_inter_multi}, we have
\begin{align}
     \E \|\theta_{t+1} - \theta^\star \|^2 &\leq \E \|\theta_{t} - \theta^\star \|^2 - 2\gamma (1-\eta) \min\left\{\frac{1}{\mu},\frac{\rho^\star - 1}{1+\mu}\right\} \E[\hat{\ell}(x_t,y_t, \theta_t)] \notag \\
   &\qquad  +2 \gamma \eta(1+R\|\theta_\star\|)  \E[\hat{\ell}(x_t,y_t, \theta_t)]   + \gamma^2 (1+BR) R^2 \E[\hat{\ell}(x_t,y_t, \theta_t)]
\end{align}
Now we consider two cases.
\begin{enumerate}
    \item[i] When $$\eta < \frac{\min\left\{ \frac{1}{\mu}, \frac{\rho^\star -1}{1+\mu}\right\}}{1+\min\left\{ \frac{1}{\mu}, \frac{\rho^\star -1}{1+\mu}\right\} + R\|\theta_\star\|},$$
    then $(1-\eta) \min\left\{\frac{1}{\mu},\frac{\rho^\star - 1}{1+\mu}\right\} > \eta(1+R\|\theta_\star\|) $. Hence, 
    \begin{align*}
         &\E \|\theta_{t+1} - \theta^\star \|^2 \leq \E \|\theta_{t} - \theta^\star \|^2 - 2\gamma (1-\eta) \min\left\{\frac{1}{\mu},\frac{\rho^\star - 1}{1+\mu}\right\} \E[\hat{\ell}(x,y, \theta_t)] \notag \\
   &\qquad  \qquad  \qquad  +2 \gamma \eta(1+R\|\theta_\star\|)  \E[\hat{\ell}(x,y, \theta_t)]   + \gamma^2 (1+BR) R^2 \E[\hat{\ell}(x,y, \theta_t)] \\
   \Rightarrow~& \E[\hat{\ell}(x,y, \theta_t)] \leq \frac{ \E \|\theta_{t} - \theta^\star \|^2 - \E \|\theta_{t+1} - \theta^\star \|^2}{2\gamma (1-\eta) \min\left\{\frac{1}{\mu},\frac{\rho^\star - 1}{1+\mu}\right\} - 2 \gamma \eta(1+R\|\theta_\star\|)  - \gamma^2 (1+BR) R^2 }
    \end{align*}
    Now in the above equation, summing for all $t$ from 1 to $n$ and applying Jensen's inequality after choosing the optimal step size $\gamma = \frac{(1-\eta)\min\left\{\frac{1}{\mu}, \frac{\rho_\star -1}{1+\mu}\right\} - \eta (1+R\|\theta_\star\|)}{R^2(1+BR)}$, we get
\begin{align}
    \E[\hat{\ell}(x,y, \bar{\theta}_n)]  \leq \frac{R^2(1+BR) \|\theta_1 - \theta_\star\|^2}{\left[ (1-\eta)\min\left\{\frac{1}{\mu}, \frac{\rho_\star -1}{1+\mu}\right\} - \eta (1+R\|\theta_\star\|)\right]^2n}.
\end{align}
\item[ii] When $$\eta \geq  \frac{\min\left\{ \frac{1}{\mu}, \frac{\rho^\star -1}{1+\mu}\right\}}{1+\min\left\{ \frac{1}{\mu}, \frac{\rho^\star -1}{1+\mu}\right\} + R\|\theta_\star\|},$$
    then $(1-\eta) \min\left\{\frac{1}{\mu},\frac{\rho^\star - 1}{1+\mu}\right\} \leq \eta(1+R\|\theta_\star\|) $. Hence, 
    
    \begin{align*}
         &\E \|\theta_{t+1} - \theta^\star \|^2 \leq \E \|\theta_{t} - \theta^\star \|^2 - 2\gamma (1-\eta) \min\left\{\frac{1}{\mu},\frac{\rho^\star - 1}{1+\mu}\right\} \E[\hat{\ell}(x,y, \theta_t)] \notag \\
   &\qquad  \qquad  \qquad  +2 \gamma \eta(1+R\|\theta_\star\|)  \E[\hat{\ell}(x,y, \theta_t)]   + \gamma^2 (1+BR) R^2 \E[\hat{\ell}(x,y, \theta_t)] \\
   \Rightarrow~ &\E[\hat{\ell}(x,y, \theta_t)] \leq \frac{\E \|\theta_{t} - \theta^\star \|^2 - \E \|\theta_{t+1} - \theta^\star \|^2}{2\gamma (1-\eta) \min\left\{ \frac{1}{\mu}, \frac{\rho^\star - 1}{1+\mu}\right\} - \gamma^2 R^2(1+BR)} \notag \\
   &\qquad  \qquad  \qquad \qquad \qquad  + \frac{2 \gamma \eta(1+R\|\theta_\star\|) }{2\gamma (1-\eta) \min\left\{ \frac{1}{\mu}, \frac{\rho^\star - 1}{1+\mu}\right\} - \gamma^2 R^2(1+BR)} \E[\hat{\ell}(x,y, \theta_t)] 
    \end{align*}
    We update via projected stochastic gradient descent. Hence,
    $$\hat{\ell}(x,y, \theta_t) \leq 1+ BR$$
    for all $(x,y)$ pair. This gives,
    \begin{align*}
        \E[\hat{\ell}(x,y, \theta_t)] &\leq \frac{\E \|\theta_{t} - \theta^\star \|^2 - \E \|\theta_{t+1} - \theta^\star \|^2}{2\gamma (1-\eta) \min\left\{ \frac{1}{\mu}, \frac{\rho^\star - 1}{1+\mu}\right\} - \gamma^2 R^2(1+BR)} \notag \\
   &\qquad  \qquad  \qquad  + \frac{2 \gamma \eta(1+R\|\theta_\star\|) }{2\gamma (1-\eta) \min\left\{ \frac{1}{\mu}, \frac{\rho^\star - 1}{1+\mu}\right\} - \gamma^2 R^2(1+BR)} (1+BR).
    \end{align*}
    Now choosing $\gamma = \frac{(1-\eta) \min\left\{ \frac{1}{\mu}, \frac{\rho^\star - 1}{1+\mu}\right\}}{R^2(1+BR)}$ gives 
    
    \begin{align}
        \E[\hat{\ell}(x,y, \theta_t)] &\leq \frac{R^2(1+BR)\left[\E \|\theta_{t} - \theta^\star \|^2 - \E \|\theta_{t+1} - \theta^\star \|^2\right]}{\left[(1-\eta) \min\left\{ \frac{1}{\mu}, \frac{\rho^\star - 1}{1+\mu}\right\}\right]^2} +\left( \frac{2(1+R\|\theta_\star\|)(1+BR)}{(1-\eta) \min\left\{ \frac{1}{\mu}, \frac{\rho^\star - 1}{1+\mu}\right\}}\right)~\eta
    \end{align}
    
Summing the above equation for $t =1$ to $n$ and applying Jensen's inequality give us, 
    \begin{align}
        \E[\hat{\ell}(x,y, \bar{\theta}_n)] &\leq \frac{R^2(1+BR) \|\theta_{1} - \theta^\star \|^2  }{\left[(1-\eta) \min\left\{ \frac{1}{\mu}, \frac{\rho^\star - 1}{1+\mu}\right\}\right]^2 n}   + \underbrace{\left( \frac{2(1+R\|\theta_\star\|)(1+BR)}{(1-\eta) \min\left\{ \frac{1}{\mu}, \frac{\rho^\star - 1}{1+\mu}\right\}}\right)}_{:\approx \text{constant}}~\eta \notag \\
    &= \frac{R^2(1+BR) \|\theta_{1} - \theta^\star \|^2  }{\left[(1-\eta) \min\left\{ \frac{1}{\mu}, \frac{\rho^\star - 1}{1+\mu}\right\}\right]^2 n} + O(\eta).
    \end{align}
\end{enumerate}
\end{proof}